\documentclass[prodmode,acmjacm]{acmsmall} 

\usepackage[ruled]{algorithm2e}

\usepackage{hyperref}

\newcommand{\eat}[1]{}
\usepackage{amsmath,amssymb,amsfonts,latexsym,url,xspace,algorithmic}
\usepackage[small]{caption}
\usepackage{float}
\usepackage{balance}
\newfloat{fragment}{H}{lop}
\floatname{fragment}{Algorithm}

\usepackage{times,epsfig,ifthen}
\usepackage{latexsym}
\usepackage{amsmath}
\usepackage{amstext}
\usepackage{amsfonts}
\usepackage{amsopn}

\usepackage{ifthen}

\newcommand{\cleaned}{\mathrm{cleaned}}
\newcommand{\ignore}[1]{}
\newcommand{\rE}{{\mathbf{E}}}
\newcommand{\E}{{\rE}}

\newcommand{\sgn}{{\mathrm{ sign}}}
\newcommand{\sign}{\sgn}

\newcommand{\R}{{\bf R}}

\newcommand{\anglesep}{\theta}

\newcommand{\reals}{R}

\newcommand{\mycomment}[1]{}

\newcommand{\err}{\mathrm{err}}

\newcommand{\rounds}{s}
\newcommand{\integers}{Z^+}

\newcommand{\poly}{\mbox{poly}}

{\begin{Sbox}\begin{minipage}{4in}\vspace*{0.2cm}}%
{\vspace*{0.2cm}\end{minipage}\hspace*{0.4cm}\end{Sbox}\fbox{\TheSbox}}

{
 \rm
\begin{tabbing}
....\=...\=...\=...\=...\=  \+ \kill
}%
{\end{tabbing}
}

\title{The Power of Localization \\
       for Efficiently Learning Linear Separators with Noise}

\author{
Pranjal Awasthi
\affil{Rutgers University}
Maria Florina Balcan
\affil{Carnegie Mellon University}
Philip M. Long
\affil{Sentient Technologies}
}

\begin{document}


\markboth{Awasthi, Balcan and Long}{Learning Linear Separators with Noise}

\begin{abstract}
We introduce a new approach for designing computationally efficient learning algorithms that are tolerant to noise,  and
demonstrate its effectiveness by designing
algorithms with improved noise tolerance guarantees for learning linear separators.

We consider both the malicious noise model of
Valiant~\cite{Valiant85,kearns-li:93} and the adversarial label noise
model of Kearns, Schapire, and Sellie~\citeyear{KearnsSS94}. For malicious noise, where the adversary can corrupt
both the label and the features, we provide a polynomial-time
algorithm for learning linear separators in $\Re^d$ under isotropic
log-concave distributions that can tolerate a nearly
information-theoretically optimal noise rate of $\eta =
\Omega(\epsilon)$, improving on the
$\Omega\left(\frac{\epsilon^3}{\log^2(d/\epsilon)}\right)$
noise-tolerance of \cite{KLS09}.  In the case that the distribution is
uniform over the unit ball, this improves on the
$\Omega\left(\frac{\epsilon}{d^{1/4}}\right)$ noise-tolerance of
\cite{KKMS08} and the
$\Omega\left(\frac{\epsilon^2}{\log(d/\epsilon)}\right)$ of
\cite{KLS09}.
For the {\em adversarial label noise} model, where the distribution
over the feature vectors is unchanged, and the overall probability of
a noisy label is constrained to be at most $\eta$, we also give a
polynomial-time algorithm for learning linear separators in $\Re^d$
under isotropic log-concave distributions that can handle a noise rate
of $\eta = \Omega\left(\epsilon\right)$.  In the case of the uniform
distribution, this improves over the results of~\cite{KKMS08} which
either required runtime super-exponential in $1/\epsilon$ (ours is
polynomial in $1/\epsilon$) or tolerated less noise.\footnote{Our
  results also improve on the tolerable noise rate from an earlier
  version of this paper \cite{ABL14}, closing log factor gaps.}

Our algorithms are also efficient in the active learning setting,
where learning algorithms only receive
the classifications of examples when they ask for them. We show
that, in this model, our algorithms achieve a label complexity whose
dependence on the error parameter $\epsilon$ is polylogarithmic
(and thus exponentially better than that of any passive algorithm).
This provides the first polynomial-time
active learning algorithm for learning linear
separators in the presence of malicious noise or adversarial label noise.

Our algorithms and analysis combine several ingredients including aggressive localization, minimization of a progressively rescaled hinge loss,
and a novel localized and soft outlier removal procedure.
We use localization techniques (previously used for obtaining  better sample complexity results) in order to obtain better noise-tolerant polynomial-time algorithms.
\end{abstract}

\category{F.2}{Theory of Computation}{Analysis of Algorithms and Problem Complexity}{}

\terms{Algorithms,Theory}


\acmformat{Pranjal Awasthi, Maria Florina Balcan and Philip M. Long, 2016. The Power of Localization for Efficiently Learning Linear Separators with Noise.}


\begin{bottomstuff}
%
Author's email addresses: P. Awasthi, {pranjal.awasthi@rutgers.edu};
M. F. Balcan, {ninamf@cs.cmu.edu}; P. M. Long, {phil.long@sentient.ai}.
\end{bottomstuff}

\maketitle

\section{Introduction}
\noindent{\bf Overview.}~~Dealing with noisy data is one of the main challenges in machine
learning and is an
active area of research.  In this work we
study the noise-tolerant learning of linear separators, arguably the most
popular class of functions used in practice~\cite{Cristianini00}.
Learning
linear separators from correctly labeled (non-noisy) examples is a
very well understood problem with simple efficient algorithms
that are effective both in the classical passive learning
setting~\cite{KearnsVazirani:94,vapnik:98} and in the more modern active learning framework~\cite{sanjoy11-encyc}.
However, for noisy settings, except for the special case of
uniform random noise, very few positive algorithmic results exist even
for passive learning. In the context of theoretical computer science more broadly, problems of noisy learning are related to seminal results in approximation-hardness~\cite{ABSS97,GR06},
cryptographic assumptions~\cite{BFKL93,Reg05}, and are connected
to other classical questions in learning theory (e.g., learning DNF formulas~\cite{KearnsSS94}), and appear as
barriers in differential privacy ~\cite{Gupta11}.

In this paper we present new techniques for designing efficient algorithms for learning linear separators in the presence of {\em malicious noise} and
{\em adversarial label noise}.
These models were originally proposed for a setting in which the
algorithm must work for an arbitrary, unknown distribution.  As we
will see, bounds on the amount of noise tolerated for this
distribution-free setting were weak, and no significant progress was
made for many years.  This motivated research investigating the role
of the distribution generating the data on the tolerable level of
noise: a breakthrough result of~\cite{KKMS08} and subsequent work
of~\cite{KLS09} showed that indeed better bounds can be obtained for
the uniform and isotropic log-concave distributions.  In this paper,
we continue this line of research.  For the malicious noise case,
where the adversary can corrupt both the label and the features
of the observation (and it has unbounded computational power and
access to the entire history of the learning algorithm's computation),
we design an efficient algorithm that can learn with accuracy
$1-\epsilon$ while tolerating an $\Omega(\epsilon)$ noise rate.
This is within a constant factor of the statistical
limit even in the case of the uniform distribution. In particular,
unlike previous works, our noise tolerance limit has no dependence on
the dimension $d$ of the space. We also show similar improvements for
adversarial label noise, and furthermore show that our algorithms can
naturally exploit the power of active learning. Active learning is a
widely studied modern learning paradigm, where the learning algorithm
only receives the class labels of examples when it asks for them. We
show that in this model, our algorithms achieve a label complexity
whose dependence on the error parameter $\epsilon$ is exponentially
better than that of any passive algorithm. This provides the first
polynomial-time active learning algorithm for learning linear
separators in the presence of adversarial label noise, solving an open
problem posed in~\cite{Balcan06,Monteleoni06}. It also provides the
first analysis showing the benefits of active learning over passive
learning under the challenging malicious noise model.

Our work brings a new set of algorithmic and analysis techniques including localization
(previously used for obtaining  better sample complexity results)
and soft outlier removal that we believe will have other applications in learning theory and optimization.
Localization \cite{BBM05,bbl05,Zha06,Balcan07,BLL09,Kol10,hanneke:11,BL13}
refers to the practice of progressively
narrowing the focus of a learning algorithm to an increasingly restricted
range of possibilities (which are known to be safe given the information
up to a certain point in time), thereby improving the stability
of estimates of the quality of
these possibilities based on random data.

In the following we start by formally defining the learning models we consider.
We then present the most relevant prior work, and then our main results and techniques.

\smallskip
\noindent{\bf Passive and Active Learning. Noise Models.}~~
In this work we consider the problem of learning  linear separators in two learning paradigms: the classical passive learning setting and the
more modern active learning scenario.
As is typical~\cite{KearnsVazirani:94,vapnik:98}, we assume that there exists a distribution  $D$ over $\Re^d$ and a fixed unknown target function whose
parameter vector is $w^*$.
In the noise-free case, in the {\em passive supervised learning} model the
algorithm is given access to a distribution oracle $EX(D,w^*)$ from
which it can get training samples $(x,\sgn(w^* \cdot x))$ where $x \sim D$. The
goal of the algorithm is to output a hypothesis $w$ such that
$err_D(w) = \Pr_{x \sim D}[ \sgn(w^* \cdot x)\ne \sgn(w \cdot x)] \le \epsilon$.
%
 In the active learning model~\cite{Cohn94,sanjoy11-encyc} the learning algorithm is
 given as input a pool of unlabeled examples drawn from the
 distribution oracle. The algorithm can then query for the labels of
 examples of its choice from the pool. The goal is to produce a
 hypothesis of low error while also optimizing for the number of
 label queries~(also known as {\em label complexity}).
 The hope is that in the active learning setting we can output a classifier
of small error by using many fewer label requests than in the passive learning setting by actively
directing the queries to informative examples (while keeping the number of unlabeled examples
polynomial).

In this work we focus on two noise models. The first one is the malicious noise model of~\cite{Valiant85,kearns-li:93} where samples are generated as follows:
with probability $(1-\eta)$ a random pair $(x,y)$ is output where $x \sim D$
and $y = \sgn(w^* \cdot x)$;
with probability $\eta$ the adversary can output an arbitrary pair
$(x,y) \in \Re^d \times \{ -1, 1 \}$.
We will call $\eta$ the noise rate.  Each of
the adversary's examples
can depend on the state of the learning algorithm and also the previous draws of the adversary. We will denote the malicious oracle as $EX_{\eta}(D,w^*)$.
The goal  remains, however,
to output a hypothesis $w$ such that $\Pr_{x \sim D}[ \sgn(w^* \cdot
x)\ne \sgn(w \cdot x)] \le \epsilon$.

In this paper, we consider an extension of the malicious noise model
to the
the active learning model as follows.  There are two oracles, an example
generation oracle and a label revealing oracle.
The example
generation oracle works as usual in the malicious noise model:
with probability $(1-\eta)$ a random pair $(x,y)$ is generated where $x \sim D$
and $y = \sgn(w^* \cdot x)$;
with probability $\eta$ the adversary can output an arbitrary pair
$(x,y) \in \Re^d \times \{ -1, 1 \}$.
In the active learning setting, unlike the standard malicious noise model,
when an example $(x,y)$ is generated, the algorithm
only receives $x$, and must make
a separate call to the label revealing oracle to get $y$.
The goal of the algorithm is still to output a
hypothesis $w$ such that $\Pr_{x \sim D}[ \sgn(w^* \cdot x)\ne \sgn(w \cdot x)] \le \epsilon$.

In the adversarial label noise model, before any examples are generated,
the adversary may choose a joint distribution $P$ over
$\Re^d \times \{ -1, 1 \}$ whose marginal distribution over $\Re^d$
is $D$ and such that
$\Pr_{(x,y) \sim P} (\sign(w^* \cdot x) \neq y) \leq \eta$.
In the active learning version of this model, once again we will have
two oracles, an example
generation oracle and a label revealing oracle.
We note that the results from our theorems in this model translate immediately into similar guarantees for the agnostic model of~\cite{KearnsSS94} (used commonly both in passive and active learning (e.g.,~\cite{KKMS08,Balcan06,Hanneke07}) -- see Appendix~\ref{a:agnostic} for details.

We will be interested in algorithms that run in time $poly(d,1/\epsilon)$ and use $poly(d,1/\epsilon)$ examples. In addition, for the active learning scenario we want our algorithms to also optimize for the number of label requests. In particular, we want the number of labeled examples to depend only polylogarithmically in $1/\epsilon$. The goal then is to quantify for a given value of $\epsilon$, the tolerable noise rate $\eta(\epsilon)$ which would allow us to design an efficient (passive or active) learning algorithm.

\smallskip
\noindent{\bf Previous Work.}
In the context of passive learning, Kearns and Li's analysis
\citeyear{kearns-li:93} implies that halfspaces
can be efficiently learned with respect to arbitrary distributions in
polynomial time while tolerating a malicious noise rate of
$\tilde{\Omega}\left(\frac{\epsilon}{d} \right)$.
Kearns and Li \citeyear{kearns-li:93} also showed that
malicious noise at a rate greater than $\frac{\epsilon}{1 + \epsilon}$
cannot be tolerated (and a slight variant of their construction
shows that this remains true even
when the distribution is uniform over the unit sphere).
The $\tilde{\Omega}\left(\frac{\epsilon}{d} \right)$ bound for
the distribution-free case was not improved for many years.
Kalai et
al.~\citeyear{KKMS08} showed that,\footnote{These results from
\cite{KKMS08} are most closely related to our work.  We describe
some of their other results, more prominently featured in their
paper, later.}
when
the distribution is uniform,
the $\poly(d,1/\epsilon)$-time averaging
algorithm tolerates malicious noise at a rate $\Omega(\epsilon/\sqrt{d})$. They also described an improvement to
$\tilde{\Omega}(\epsilon/{d^{1/4}})$ based on the observation that
uniform examples will tend to be well-separated, so that pairs of
examples that are too close to one another can be removed, and this
limits an adversary's ability to coordinate the effects of its noisy
examples.~\cite{KLS09} analyzed another approach to limiting the
coordination of the noisy examples: they proposed an outlier removal
procedure that used PCA to find any direction $u$ onto
which projecting the training data led to suspiciously high variance,
and removing examples with the most extreme values after projecting
onto any such $u$.  Their algorithm tolerates malicious noise at a rate
$\Omega(\epsilon^2/\log(d/\epsilon))$ under the uniform distribution.

Motivated by the fact that many modern machine learning applications
have massive amounts of unannotated or unlabeled data, there has been significant interest in designing active learning algorithms that most efficiently utilize the available data, while minimizing the need for human intervention.
Over the past decade there has been substantial progress on understanding the underlying statistical principles of active learning, 
and several general characterizations have been developed for
describing when active learning could have an advantage over the classical passive supervised learning paradigm both in the noise free settings and in the agnostic case~\cite{QBC,sanjoy-coarse,Balcan06,Balcan07,Hanneke07,dhsm,CN07,BHW08,Kol10,nips10,wang11,sanjoy11-encyc,RaginskyR:11,BH12,hanneke:survey}.
However, 
despite many efforts, except for very simple noise models (random classification noise~\cite{BF13} and linear noise~\cite{dgs12}),
to date there are no known computationally efficient algorithms with provable guarantees in the presence of noise. In particular, there are no computationally efficient algorithms for
the agnostic case, and  furthermore no result exists showing the benefits of active learning over passive learning in the malicious noise model, where the
adversary may also corrupt the features.

We discuss additional related work in Appendix~\ref{se:related}.

\subsection{Our Results}
The following are our main results.

\begin{theorem}
\label{thm:log-concave-agnostic}
There is a polynomial-time algorithm
$A_1$ for learning linear
separators with respect to isotropic log-concave distributions
in $\Re^d$ in the presence of adversarial label noise, and positive constants
$C$ and $\epsilon_0$ such that, for all $0 < \epsilon < \epsilon_0$,
and all $\delta > 0$,
if $\eta < C \epsilon$,
then the output $w$ of $A_1$
satisfies $\Pr_{(x,y) \sim D} [\sgn(w\cdot x) \neq \sgn(w^*\cdot x)] \leq \epsilon$
with probability at least $1-\delta$.

Further, $A_1$ uses at most poly($d$, $\log(1/\epsilon)$, $\log(1/\delta)$)
labeled examples.
\end{theorem}

\begin{theorem}
\label{thm:log-concave-malicious}
There is a polynomial-time algorithm
$A_2$ for learning linear
separators with respect to isotropic log-concave distributions
in $\Re^d$ in the presence of malicious noise, and positive constants
$C$ and $\epsilon_0$ such that, for all $0 < \epsilon < \epsilon_0$,
and all $\delta > 0$,
if $\eta < C \epsilon$,
then the output $w$ of $A_2$
satisfies $\Pr_{(x,y) \sim D} [\sgn(w\cdot x) \neq \sgn(w^*\cdot x)] \leq \epsilon$
with probability at least $1-\delta$.

$A_2$ uses at most poly($d$, $\log(1/\epsilon)$, $\log(1/\delta)$)
labeled examples.
\end{theorem}

As a restatement of Theorem~\ref{thm:log-concave-agnostic}, in the agnostic setting
considered in \cite{KKMS08}, we can output a halfspace of error at
most $O(\eta + \alpha)$ in time $\poly(d,1/\alpha)$.  In the case of
the uniform distribution, Kalai, et al, achieved error $\eta + \alpha$
by learning a low degree polynomial in time whose dependence on the
inverse accuracy is super-exponential.  On the other hand, this result
of \cite{KKMS08} applies when the target halfspace does not necessarily
go through the origin.

Our algorithms naturally exploit the power of active learning.
(Indeed, as we will see, an active learning algorithm proposed in
\cite{Balcan07} provided the springboard for our work.)
We show
that in this model, the label complexity of both algorithms is
polylogarithmic in $1/\epsilon$.
%
Our efficient algorithm that tolerates adversarial label
noise solves an open problem posed in~\cite{Balcan06,Monteleoni06}.
Furthermore, our paper provides the first active learning algorithm
for learning linear separators in the presence of non-trivial amount
of adversarial noise that can affect not only the label, but also
the
features.

Our work exploits the power of localization
for designing noise-tolerant polynomial-time algorithms.
Such localization techniques have been used for
analyzing sample complexity for  passive learning
(see \cite{BBM05,bbl05,Zha06,BLL09,BL13})
or for designing active learning algorithms
(see \cite{Balcan07,Kol10,hanneke:11,BL13}).
Ideas useful for making such a localization strategy computationally efficient,
and tolerating malicious noise, are described in Section~\ref{se:tech}.

We note that all our algorithms are proper in that they return a
linear separator.  (Linear models can be evaluated efficiently, and
are otherwise easy to work with.)  We summarize our results, and the
most closely related previous work, in
Tables~\ref{table:comparison-uniform} and~\ref{table:log-concave}.

\begin{table}[ht]
\caption{Comparison with previous $\mathrm{poly}(d,1/\epsilon)$-time
         algs. for uniform distribution}
\centering
\begin{tabular}{|c|c|c|}
\hline
\textbf{Passive Learning} & Prior work & Our work\\
 \hline
malicious & $\eta = \Omega\left(\frac {\epsilon} {d^{1/4}}\right)$~\cite{KKMS08} & $\eta = \Omega\left(\epsilon\right)$\\
& $\eta = \Omega\left(\frac {\epsilon^2} {\log\left(d/\epsilon\right) }\right)$~\cite{KLS09} & \\
 \hline
 adversarial & $\eta = \Omega\left(\frac{\epsilon}{\sqrt{\log\left(1/\epsilon\right)}}\right)$
            ~\cite{KKMS08} & $\eta = \Omega\left(\epsilon\right)$\\
\hline
\textbf{Active Learning} & NA & $\eta = \Omega\left(\epsilon\right)$\\
 (malicious and adversarial) & & \\ \hline
\end{tabular}
\label{table:comparison-uniform}
\end{table}
\begin{table}[ht]
\caption{Comparison with previous $\mathrm{poly}\left(d,1/\epsilon\right)$-time
         algorithms isotropic log-concave distributions}
\centering
\begin{tabular}{|c|c|c|}
\hline
\textbf{Passive Learning} & Prior work & Our work\\
 \hline
malicious & $\eta = \Omega\left(\frac {\epsilon^3} {\log^2\left(d/\epsilon\right) }\right)$~\cite{KLS09} & $\eta = \Omega\left(\epsilon \right)$\\
 \hline
adversarial &  $\eta = \Omega\left(\frac {\epsilon^3} {\log\left(1/\epsilon\right) }\right)$~\cite{KLS09} & $\eta = \Omega\left(\epsilon \right)$\\
\hline
\textbf{Active Learning} & NA & $\eta = \Omega\left(\epsilon \right)$\\
 (malicious and adversarial) & & \\ \hline
\end{tabular}
\label{table:log-concave}
\end{table}


\subsection{Techniques}
\label{se:tech}

\smallskip
\noindent{\bf Hinge Loss Minimization}
As minimizing the 0-1 loss in the presence of noise is NP-hard~\cite{Johnson78,Garey90},
a natural approach is to minimize a surrogate convex loss that acts as a proxy for the 0-1 loss.
A common choice in machine learning is to use the hinge loss:
$
\max\left(0,1 - y(w\cdot x) \right).
$
In this paper, we use the slightly more general
$
\ell_{\tau}(w,x,y) = \max\left(0,1 - \frac{y(w\cdot x)} \tau\right),
$
and, for a set $T$ of examples, we let
$
\ell_{\tau}(w,T) = \frac{1}{|T|} \sum_{(x,y) \in T} \ell_{\tau}(w,x,y).
$
Here $\tau$ is a parameter that changes during training.
It can be shown that minimizing hinge loss with an appropriate normalization factor
can tolerate a noise rate of $\Omega(\epsilon^2/\sqrt{d})$ under isotropic log-concave distributions in $\Re^d$. This is also the limit for such a strategy since a more powerful malicious adversary can concentrate all the noise directly opposite to the target vector $w^*$ and make sure that the hinge-loss is no longer a faithful proxy for the 0-1 loss.

\smallskip
\noindent{\bf Localization in the instance and concept space}~~~
Our first key insight is that by using an iterative localization technique, we can limit the harm caused by an adversary at each stage and hence can still do hinge-loss minimization despite significantly more noise.
In particular, the iterative algorithm we propose
proceeds
in stages and at stage $k$, we have a hypothesis vector $w_k$ of a
certain error rate.
The goal in stage $k$ is to produce a new vector $w_{k+1}$ with error
rate a constant factor smaller than $w_k$'s.  In order to reduce the
error rate, we focus on a band of size $b_{k} = e^{-c k}$ around the
boundary of the linear classifier whose normal vector is $w_k$,
i.e. $S_{w_k, b_{k}} = \{x: |w_{k} \cdot x| < b_{k}\}$.  For the rest
of the paper, we will repeatedly refer to this key region of
borderline examples as ``the band''.  The key observation made
in~\cite{Balcan07} is that outside the band, all the classifiers still
under consideration (namely those hypotheses within radius $r_k$ of
the previous weight vector $w_{k}$) will have very small
error. Furthermore, the probability mass of this band under the
original distribution is small enough, so that in order to make the
desired progress we only need to find a hypothesis of constant error
rate over the data distribution conditioned on being within margin
$b_{k}$ of $w_{k}$. This idea was used in~\cite{Balcan07} to obtain
active learning algorithms with improved label complexity ignoring
computational complexity considerations\footnote{We note that the
  localization considered by~\cite{Balcan07} is a more aggressive one
  than those considered in disagreement based active learning
  literature~\cite{Balcan06,Hanneke07,Kol10,hanneke:11,wang11} and
  earlier in passive learning~\cite{BBM05,bbl05,Zha06}.}.

In this work, we build on this idea to produce polynomial time algorithms with improved noise tolerance.
To obtain our results, we exploit several new ideas: (1) the performance of the rescaled hinge loss minimization in smaller and smaller bands,
(2) an analysis of properties of the distribution obtained after
conditioning on the band that enables us to more sensitively identify
cases in which the adversary concentrates the effects of noisy examples,
(3) another type of localization  --- a novel soft outlier removal procedure.

We first show that if we minimize a variant of the hinge loss that is
rescaled depending on the width of the band, it remains a faithful
enough proxy for the 0-1 error even when there is significantly more
noise.  As a first step towards this goal, consider the setting where
we pick $\tau_k$ proportionally to $b_k$, the size of the band, and
$r_k$ is proportional to the error rate of $w_{k}$, and then minimize
a normalized hinge loss function $\ell_{\tau_k}(w,x,y) = \max (0, 1 -
\frac{y(w \cdot x)}{\tau_k})$ over vectors $w$ in $B(w_{k},r_k)$, the
ball of radius $r_k$ centered at $w_k$.  We first show that $w^*$ has
small hinge loss within the band. Furthermore, within the band the
adversarial examples cannot hurt the hinge loss of $w^*$ by a lot. To
see this notice that if the malicious noise rate is $\eta$, within
$S_{w_{k-1},b_k}$ the effective noise rate is ${O}(\eta/b_k)$.  Also,
with high probability, the hinge loss for vectors $w \in B(w_{k},r_k)$
is at most $\tilde{O}(\sqrt{d})$.  Hence the maximum amount by which
the adversary can affect the hinge loss is $\tilde{O}(\eta
\sqrt{d}/b_k)$. Using this approach we get a noise tolerance of
$\tilde{\Omega}(\epsilon/\sqrt{d})$.

In order to get better tolerance in the adversarial, or agnostic, setting, we note that examples $x$ for which $|w \cdot x|$ is large for $w$ close to
$w_{k-1}$ are the most harmful, and, by analyzing the variance of
$w \cdot x$ for such directions $w$, we can more effectively limit the amount by which an adversary can ``hurt'' the hinge loss. This  then leads to an improved noise tolerance of $\Omega(\epsilon)$.

Our algorithm that tolerates adversarial label noise does not work for the
malicious noise model: it can be foiled by an algorithm that concentrates
$\eta$ measure on an incorrectly labeled example within $\Theta(\epsilon)$
of the separating hyperplane of the target, but with a very large norm.
If the norm of this noisy example is large enough, its hinge loss
can overwhelm the hinge losses of clean examples.
We cope with this using a {\em soft localized outlier removal} procedure at each stage (described next). This procedure assigns a weight to each data point indicating
the algorithm's confidence that the point is not
``noisy''. We then minimize the weighted hinge loss. Combining this with the variance analysis mentioned above leads to a noise of tolerance of $\Omega(\epsilon)$ in the malicious case.

\smallskip
\noindent{\bf Soft Localized Outlier Removal}
Outlier removal has been used for learning linear classifiers before \cite{BlumFKV:97,KLS09}.
In \cite{KLS09}, the goal of
outlier removal was to limit the ability of the adversary to coordinate
the effects of noisy examples -- excessive such coordination was
detected and removed. Our outlier removal procedure~(Algorithm~\ref{fig:outlier-removal}) is similar in spirit to that of \cite{KLS09} with two key differences.
First, as in \cite{KLS09}, we will use the
variance of the examples in a particular direction to measure their
coordination.  However, due to the fact that in round $k$, we are minimizing the hinge loss only with
respect to vectors that are close to $w_{k-1}$, we only need to limit the
variance in these directions.
As training proceeds, the band is increasingly shaped
like a pancake, with $w_{k-1}$ pointing in its flattest direction.   Hypotheses
that are close to $w_{k-1}$ also point in flat directions;
the variance in those directions is $\Theta(b^2_k)$ which is much smaller
than variance found in a generic direction.
This allows us to limit the harm of the adversary to a greater extent than was
possible in the analysis of \cite{KLS09}.
The second difference is that, unlike previous outlier removal techniques,
rather than making discrete remove-or-not decisions, we instead weigh the
examples and then minimize the weighted hinge loss. Each weight indicates the algorithm's
confidence that an example is not noisy.
We show that these weights can be computed by solving a linear program with infinitely many constraints. We then show how to design an efficient separation oracle for the linear program using recent general-purpose optimization techniques
\cite{SZ03,BM14}.

\subsection{Recent developments}
Subsequent to the publication of this work in preliminary form
\cite{ABL14}, Daniely \citeyear{Dan15} combined the techniques of
this paper with the polynomial-separation technique of \cite{KKMS08}
to achieve a PTAS for agnostic learning of halfspaces with
respect to the uniform distribution.  (Recall that agnostic
learning is essentially equivalent to learning with
adversarial label noise, as outlined in Appendix~\ref{a:agnostic}.)
Awasthi et al.~\citeyear{ABHU15}
provided  efficient (active and passive) learning algorithms for learning linear separators in
the presence of (sufficiently benign) bounded noise (a.k. a. Massart noise)\footnote{The Massart noise is widely studied in statistical learning theory(see e.g.~\cite{bbl05}) and can be thought of as a realistic generalization of the random classification noise, where where  the  label  of  each
example $x$ is flipped independently with constant probability $\eta(x) < 1/2$.} to arbitrarily small excess error  under the uniform distribution over
the unit sphere in $R^d$.  
Awasthi et al.~\citeyear{ABHZ16} improved on this algorithm (to allow for any constant bounded noise),
and extended the technique to apply to the related problems of attribute efficient learning of linear separators and the popular signal processing problem of 1-bit compressed sensing (both in the passive learning model). The recent work of~\cite{diakonikolas2017learning} has extended the results of this paper to also include non-homogeneous halfspaces.

\section{Preliminaries}

Recall that
$
\ell_{\tau}(w,x,y) = \max\left(0,1 - \frac{y(w\cdot x)} \tau\right)
$
and
$
\ell_{\tau}(w,T) = \frac{1}{|T|} \sum_{(x,y) \in T} \ell_{\tau}(w,x,y).
$
Similarly, the
expected hinge loss w.r.t.\ $D$ is defined as
$L_{\tau}(w,D) = E_{x \sim D} (\ell_{\tau}(w,x,\sgn(w^* \cdot x)))$.
Our analysis will also consider the distribution $D_{w,\gamma}$
obtained by conditioning $D$ on membership in the band, i.e.\ the
set $\{x: |w\cdot x| \le \gamma \}$.

We present our algorithms in the active learning model. Since we will
prove that our active algorithm only uses a polynomial number of
unlabeled samples, this will imply a guarantee for passive learning
setting.  At a high level, our algorithms are iterative learning
algorithms that operate in rounds.  In each round $k$ we focus on
points that fall near the decision boundary of the current hypothesis
$w_{k-1}$ and use them in order to obtain a new vector $w_k$ of lower
error.  In the malicious noise case, in round $k$ we first do a soft
outlier removal and then minimize hinge loss normalized appropriately
by $\tau_k$.

When analyzing the malicious noise model, we will refer to the examples
generated by the adversary as the {\em noisy examples}, and the other
examples as the {\em clean examples}.

For vectors $u$ and $v$, denote the angle between them by
$\theta(u,v)$.  Let $B(u,r)$ be the ball of radius $r$ centered at
$u$.

 The description of the algorithms and their analysis is simplified if we assume that it starts with a preliminary weight vector
$w_0$ whose angle with
the target $w^*$ is acute, i.e.\ that satisfies
$\theta(w_0, w^*) < \pi/2$.  We show in Appendix~\ref{a:w0} that this is
without loss of generality for the types of problems we consider.

A probability distribution is {\em isotropic log-concave} if its density
can be written as $\exp(-\psi(x))$ for a convex function $\psi$, its
mean is ${\bf 0}$, and its covariance matrix is $I$.

\section{Adversarial label noise}
\label{s:label}

Algorithm~\ref{fig:active-algorithm-adversarial} is our algorithm
for learning in the presence of adversarial label noise.
In the analysis below, we assume that the algorithm has access to $w_0$
such that $\theta(w_0,w^*) < \pi/2$.  This can be shown to be
without loss of generality (see Appendix~\ref{a:w0})).

\begin{fragment*}[t]
          \textbf{Input}: allowed error rate $\epsilon$, probability of failure $\delta$,
an oracle that returns $x$, for $(x,y)$ sampled from $\mathrm{EX}_{\eta} (f,D)$,
and an oracle for getting the label from an example;
a sequence of sample sizes $m_k>0$;
a sequence of cut-off values $b_k >0$;
a sequence of hypothesis space radii $r_k >0$;
a precision value $\kappa >0$
          \vspace*{0.01in}
\begin{enumerate} \itemsep 0pt
\small
\item Draw $m_1$ labeled examples and put them into a working set $W$.
\item For $k=1,\ldots, \rounds = \lceil \log_2(1/\epsilon) \rceil$
\begin{enumerate}
\item Find ${v}_k \in B({w}_{k-1},r_k)$
to approximately minimize training hinge loss over $W$
s.t. $\Vert {v}_k \Vert_2 \leq 1$: \\
$\ell_{\tau_k}({v}_k ,W) \leq \min_{w \in B({w}_{k-1},r_k) \cap B(0,1))}
            \ell_{\tau_k}(w,W) + \kappa/8$.  \- \-
\item Normalize ${v}_k$ to have unit length, yielding
${w}_k = \frac{{v}_k}{\Vert {v}_k \Vert_2}$.
\item Clear the working set $W$.
\item \textbf{Until} $m_{k+1}$ additional data points are put in $W$, given an
unlabeled example $x$ for $(x,f(x))$ obtained from $\mathrm{EX}_{\eta} (f,D)$,
  \textbf{if} $|{w}_{k}\cdot x| \geq b_k$, \textbf{then} reject $x$
 \textbf{else} ask for the label of $x$ and put the example into $W$
\end{enumerate}
\end{enumerate}
\textbf{Output}: Weight vector $w_{\rounds}$ of error at
most $\epsilon$ with probability $1-\delta$.
\caption{\label{fig:active-algorithm-adversarial}{\sc Computationally Efficient Algorithm tolerating
          adversarial label noise}}
\end{fragment*}

Theorem~\ref{thm:log-concave-agnostic} follows immediately from the
following theorem analyzing Algorithm~\ref{fig:active-algorithm-adversarial}.
\begin{theorem}
\label{t:adversarial.detailed}
Let a distribution $D$ over $\reals^d$
be isotropic log-concave.
Let $w^*$ be the (unit length) target
weight vector.  There are settings of the parameters of
Algorithm~\ref{fig:active-algorithm-adversarial},
and positive constants $M$, $C$ and $\epsilon_0$,
such that, for all $\epsilon < \epsilon_0$, for any $\delta>0$,
if the rate $\eta$ of adversarial noise satisfies $\eta < C \epsilon$, a number $n_k =
\mathrm{poly}(d,M^k, \log(1/\delta))$ of unlabeled examples in round $k$ and a
number $m_k = O\left(d \log\left(\frac{d}{\epsilon\delta}\right) (d
+ \log(k/ \delta) )\right)$ of labeled examples in round $k \geq 1$,
and $w_0$ such that $\theta(w_0,w^*) < \pi/2$,
after $s=O(\log(1/\epsilon))$ iterations, finds
$w_{\rounds}$ satisfying $\err(w_{\rounds}) \leq \epsilon$
with probability $\geq 1-\delta$.
%
\end{theorem}

The rest of this section is dedicated to the proof of
Theorem~\ref{t:adversarial.detailed}.

\subsection{Relevant properties of isotropic log-concave distributions}

We start by listing some properties of i.l.c.\ distributions that we
will use in our analysis.
\begin{lemma}[\cite{LV07,Vem10}]
\label{l:ilc}
Assume that $D$  is isotropic
log-concave in $R^d$ and let $f$ be its density function.
\begin{enumerate}
\item[(a)]
 $\Pr_{x \sim D} {[ ||x||_2  \geq \alpha \sqrt{d}]} \leq e^{-\alpha +1}.$
 \item[(b)] Projections of $D$ onto subspaces of $\R^d$ are isotropic log-concave.
 \item[(c)] If $d=1$, then $\Pr_{x \sim D} {[ x \in [a,b]]} \leq |b-a|.$
\item[(d)] There is an absolute constant
$c_1$ such that, if $d=1$, $f(x) > c_1$ for all $x \in [-1/9,1/9]$.
\item[(e)] There is an absolute constant $c_2$
such that for any two unit vectors $u$ and $v$ in $\reals^d$  we have
$ c_2 \anglesep(v,u) \leq \Pr_{x \sim D} (\sign(u \cdot x) \neq \sign(v \cdot x)).$
\item[(f)] For any $d$, there are positive $c_3(d)$ and $c_4(d)$ such that
$f(x) \leq c_3(d) \exp(-c_4(d) || x ||)$.
\end{enumerate}
\end{lemma}
Parts (a)-(d) are from \cite{LV07}. Part (e) is implicit
in \cite{Vem10}, and set out explicitly in \cite{BL13}.
Part (f) is from \cite{KLT09}.

We will use the following lemma as a tool to analyze the
variance in directions close to the hypothesis at any given time.
\begin{lemma}
\label{lemma:tail.ilc}
For any $C > 0$,  there exist constants $c,c'$ such that,
for any isotropic log-concave distribution $D$,
for any $a$ such that, $\Vert a \Vert_2 \le 1$,
and $|| u - a ||_2 \leq r$, for any
$0 < \gamma < C$, and for any
$K \geq 4$, we have
\[
\Pr_{x \sim D_{u,\gamma}}\left(|a \cdot x| >
K \sqrt{r^2 + \gamma^2} \right) \leq ce^{-c'K\sqrt{1+\frac{\gamma^2}{r^2}}}.
\]
\end{lemma}
\begin{proof}
W.l.o.g.\ we may assume that $u = (1,0,0, \cdots, 0)$.

Let $a' = (a_2,...,a_d)$, and, for a random $x = (x_1,x_2,...,x_d)$ drawn from
$D_{u,\gamma}$, let $x' = (x_2,...,x_d)$.  We may rewrite the probability
that we want to bound as
\begin{equation}
\label{e:twoparts}
\Pr_{x \sim D_{u,\gamma}}\left(|a \cdot x| >
K \sqrt{r^2 + \gamma^2} \right)
= \frac{\Pr_{x \sim D} \left(|a \cdot x| >
                        K \sqrt{r^2 + \gamma^2}
                        \mbox{ and } |x_1| \leq \gamma \right)}
         {\Pr_{x \sim D} \left(|x_1| \leq \gamma \right)}.
\end{equation}
Lemma~\ref{l:ilc}
implies that there is a positive constant $c_1$ such that the
denominator satisfies the following lower bound:
\begin{equation}
\label{e:denominator}
\Pr_{x \sim D} \left(|x_1| \leq \gamma \right)
\geq c_1 \min \{ \gamma, 1/9 \}
\geq \frac{c_1 \gamma}{9 C}.
\end{equation}
So now, we just need an upper bound on the numerator.  We
have
\begin{eqnarray*}
&& \Pr_{x \sim D} \left(|a \cdot x| >
                        K \sqrt{r^2 + \gamma^2}
                        \mbox{ and } |x_1| \leq \gamma \right)  \leq \Pr_{x \sim D} \left(|a' \cdot x'| >
                        K \sqrt{r^2 + \gamma^2} - \gamma \mbox{ and } |x_1| \leq \gamma \right) \\
&& \leq \Pr_{x \sim D} \left(|a' \cdot x'| >
                        (K-1) \sqrt{r^2 + \gamma^2} \mbox{ and } |x_1| \leq \gamma \right).
\end{eqnarray*}
Define $a'' = \frac{a'}{\|a'\|}$. Define random variable $Y$ to be $a'' \cdot x$ and a random variable $X$ to be $x_1$ where $x$ is drawn from $D$. Then we have $E[X] = E[x_1] = 0$. $E[Y] = E[a'' \cdot x] = 0$. Furthermore, $E[X^2] = 1$, $E[Y^2] = 1$ and $E[XY] = 0$. Hence, the joint distribution of $X$ and $Y$ is isotropic log-concave. Let $f(X,Y)$ be the p.d.f. of this distribution. Then the numerator can be upper bounded as follows:
$$
4\Pr_{x \sim D} \left(a' \cdot x' >
                        (K-1) \sqrt{r^2 + \gamma^2} \mbox{ and } 0 \leq x_1 \leq \gamma \right) \leq 4\int_{0}^{\gamma} \int_{\frac{(K-1) \sqrt{r^2 + \gamma^2}}{\|a'\|}}^{\infty} f(X,Y) dX dY.
$$
Applying Part (f) of Lemma~\ref{l:ilc} with $d=2$, there are constants
$c$ and $c'$ such that the numerator is at most
\begin{eqnarray*}
&& c\int_{0}^{\gamma} \int_{\frac{(K-1) \sqrt{r^2 + \gamma^2}}{\|a'\|}}^{\infty} \exp(-c'\sqrt{X^2 + Y^2}) dX dY\\
&& \leq c\int_{0}^{\gamma} \int_{\frac{(K-1) \sqrt{r^2 + \gamma^2}}{\|a'\|}}^{\infty} \exp(-c' Y) dX dY\\
&& \leq c'' \gamma \exp(-c'\frac{(K-1) \sqrt{r^2 + \gamma^2}}{\|a'\|}),
\end{eqnarray*}
in part because the fact that $\|a'\| \leq r$ implies that
$\frac{(K-1) \sqrt{r^2 + \gamma^2}}{\|a'\|} > 3$.
Hence the numerator of (\ref{e:twoparts}) is at most $\leq c'' \gamma \exp(-c'(K-1) \sqrt{1+\frac{\gamma^2}{r^2}})$, completing the proof.
\end{proof}

Armed with Lemma~\ref{lemma:tail.ilc},
now we are ready for the variance bound.  It improves on a
bound from an earlier version of this paper \cite{ABL14}, matching
what was obtained in that version for the special case of the uniform
distribution.  This improvement is what leads to closing a log
factor gap in the tolerable rate of noise for i.l.c.\ distributions.
\begin{lemma}
\label{lemma:var.small}
Assume that $D$ is isotropic log-concave.

For any $c_3$, there is a constant $c_4$ such that, for
all $0 < \gamma \leq c_3$,
for all $a$ such that $\Vert u-a \Vert_2 \le r$ and $\Vert a \Vert_2 \le 1$
\[
\E_{x \sim D_{u,\gamma}} ((a \cdot x)^2)
  \leq c_4 (r^2 + \gamma^2) .
\]
\end{lemma}
{\bf Proof:} Let  $z = \sqrt{r^2 + \gamma^2}$. Setting, with foresight, $t = 16 z^2$, we have
\begin{eqnarray}
\nonumber
&& \E_{x \sim D_{u,\gamma}} ((a \cdot x)^2) \\
\nonumber
&& = \int_0^{\infty} \Pr_{x \sim D_{u,\gamma}} ((a \cdot x)^2 \geq \alpha)\; d\alpha \\
&& \leq t+ \int_t^{\infty} \Pr_{x \sim D_{u,\gamma}} ((a \cdot x)^2 \geq \alpha)\; d\alpha.
\end{eqnarray}
Since $t \geq 16 z^2$, Lemma~\ref{lemma:tail.ilc}
implies that, for absolute constants $c$ and $c'$, we have
\begin{eqnarray*}
&& \E_{x \sim D_{u,\gamma}} ((a \cdot x)^2)
\leq t
  + c \int_t^{\infty}
    \exp(-c' \frac{\sqrt{\alpha}}{r}) d\alpha.
\end{eqnarray*}

Now, we want to evaluate the integral.
Using a change of variables $u^2 = \alpha$, we get
\[
\int_t^{\infty} \exp\left( -c' \sqrt{\alpha}/r \right) d\alpha
=
2 \int_{\sqrt{t}}^{\infty} u \exp\left( -c' u/r \right) du
=  \frac{2 r^2}{{c'}^2} \left(\frac{\sqrt{t}}{r} + 1\right) \exp\left( -c'\sqrt{t}/r \right).
\]
Putting it together, we get
\[
\E_{x \sim D_{u,\gamma}} ((a \cdot x)^2)
 \leq t + {\frac{2 c r^2}{{c'}^2} \left(\frac{\sqrt{t}}{r} + 1\right)
          \exp\left( -c' \sqrt{t}/r \right)}
 \leq \left(1 + \frac{c}{2{c'}^2} \right) t
                + \frac{2 c r^2}{{c'}^2} \exp\left(-4 c'\frac z r\right)
\]
and, since $t = 16 z^2$ and $\frac z r \geq 1$, we get the desired bound.
\qed

Finally, we will use a lemma from \cite{BL13} that
generalizes and strengthens a key lemma from \cite{Balcan07}.
It is used to show that, during the learning process,
most large-margin examples are classified correctly.
\begin{lemma}[Theorem 4 of \cite{BL13}]
\label{lemma:vectors.ilc}
For any $c_5 > 0$, there is a $c_6 > 0$ such that the following holds.
Let $u$ and $v$ be two unit vectors in $R^d$, and assume that
$\theta(u,v) = \alpha < \pi/2$. If $D$ is isotropic log-concave in
$\R^d$, then $\Pr_{x \sim D} [ \mathrm{sign}(u \cdot x) \neq
  \mathrm{sign}(v \cdot x) \mbox{ and } | v \cdot x| \geq c_6 \alpha]
\leq c_5 \alpha.$
\end{lemma}


\subsection{Parameters for the algorithm}
\label{s:parameters.labelnoise}

For easy reference throughout the proof, here we collect together
the settings of the parameters of the algorithm.

Let $M = \max \{ \frac2 {c_2 \pi}, 2 \}$, where $c_2$ is from
Lemma~\ref{l:ilc}.
Let $c_1'$ be the value of $c_6$ in
Lemma~\ref{lemma:vectors.ilc}
corresponding
to the case where $c_5$ is $\frac{c_2}{4 M}$;
then let $b_k= c_1' {M}^{-k}$.

Let $c_2'$ be $c_1$ from Lemma~\ref{l:ilc}.
Let $r_k = \min \{ {M}^{-(k-1)}/c_2, \pi/2 \}$,
    where $c_2$ is from Lemma~\ref{l:ilc}
and $\kappa = \frac{1}{4 c_1' M}$.
Let $\tau_k = \frac{c_1 \min \{ b_{k-1}, 1/9 \} \kappa}{6}$,
where $c_1$ is the value from Lemma~\ref{l:ilc}.

Let $z^2_k = r^2_k + b^2_{k-1}$.

\subsection{The error within a band in each iteration}
At each iteration, Algorithm~\ref{fig:active-algorithm-adversarial}
concentrates its attention on examples in the band.  Our next theorem
analyzes its error on these examples.

\begin{theorem}
\label{thm:low-error-inside-band.labelnoise}
For $k \leq \lceil \log_M (1/\epsilon) \rceil$,
if $err_{D}(w_{k-1}) \le M^{-(k-1)}$,
with probability $1 - \frac{\delta}{k+k^2}$ (over the random examples
in round $k$),
after round $k$ of Algorithm~\ref{fig:active-algorithm-adversarial},
we have $err_{D_{w_{k-1}, b_{k-1}}}(w_k) \le \kappa$.
\end{theorem}

We will prove Theorem~\ref{thm:low-error-inside-band.labelnoise} using a series
of lemmas below. First, we bound the hinge loss of the target $w^*$
within the band $S_{w_{k-1}, b_{k-1}}$.  Since we are analyzing a
particular round $k$, to reduce clutter in the formulas, for the rest
of this section, let us refer to
$\ell_{\tau_k}$ as $\ell$, and
$L_{\tau_k}(\cdot,D_{w_{k-1}, b_{k-1}})$ as $L(\cdot)$.

\begin{lemma}
\label{lemma:opt.hinge.underlying}
$L(w^*) \leq \kappa/6$.
\end{lemma}
\begin{proof}
Notice that $y(w^*\cdot x)$ is never negative, so, on any
clean example $(x,y)$, we have
\[
\ell(w^*,x,y) = \max \left\{ 0, 1 - \frac{y(w^* \cdot x)}{\tau_k} \right\}
            \leq 1,
\]
and, furthermore, $w^*$ will pay a non-zero hinge only inside the region where
$|w^*\cdot x| < \tau_k$.
Hence,
\[
L(w^*) \leq \Pr_{D_{w_{k-1}, b_{k-1}}}(|w^*\cdot x| \leq \tau_k)
= \frac {\Pr_{x\sim D} (|w^*\cdot x| \leq \tau_k \, \& \, |{w}_{k-1}\cdot x| \leq b_{k-1})} {\Pr_{x \sim D}(|{w}_{k-1}\cdot x| \leq b_{k-1})}.
\]

Using Part (d) of
Lemma~\ref{l:ilc}, for
the value of $c_1$ in that definition,
we can lower bound the denominator:
\[
\Pr_{x \sim D}(|w_{k-1} \cdot x| < b_{k-1})
  \ge  2 c_1 \min \{ b_{k-1}, 1/9 \}.
\]
Part (c) of Lemma~\ref{l:ilc}
also implies that the numerator is at most
\[
\Pr_{x \sim D} (|w^* \cdot x| \leq \tau_k) \leq 2 \tau_k.
\]

Hence, we have
\[
L(w^*) \leq \frac{2 \tau_k}{2 c_1 \min \{ b_{k-1}, 1/9 \}} = \kappa/6.
\]
\end{proof}


%
%


Let $\tilde{P}$ be the joint distribution used by the algorithm,
which includes the noisy labels chosen by the adversary.
Let $N = \{ (x,y) : \sign(w^* \cdot x) \neq y \}$ consist of
noisy examples, so that $\tilde{P}(N) \leq \eta$.
Let $P$ be the joint distribution obtained by
applying the correct labels.  Let $\tilde{P}_k$ be the distribution
on the examples given to the algorithm in round $k$ (obtained by
conditioning $\tilde{P}$ to examples that fall within the band), and
let $P_k$ be the corresponding joint distribution with clean labels.

The key lemma here is to bound how far the expected loss with respect
to the distribution $\tilde{P}_k$ given to the algorithm is from
to the expected loss with respect to the distribution $P_k$ with the
cleaned labels.  Informally, it shows that, to an extent,
$\E_{(x,y) \in \tilde{P}_k} (\ell(w, x, y))$ is an effective
proxy for $\E_{(x,y) \in P_k} (\ell(w, x, y))$.
\begin{lemma}
\label{lemma:clean.vs.dirty.labelnoise}
There is an absolute positive constant $c$ such that,
if we define $z_k = \sqrt{r_k^2 + b_{k-1}^2}$ then
for any $w \in B({w}_{k-1},r_k)$, we have
\begin{equation}
\label{e:clean.vs.dirty}
| \E_{(x,y) \in P_k} \ell(w, x, y)) - \E_{(x,y) \in \tilde{P}_k} \ell(w, x, y)) |
 \leq c \sqrt{\frac{\eta}{\epsilon}} \frac{ z_k }{\tau_k}.
\end{equation}
\end{lemma}
\begin{proof}
Fix an arbitrary $w \in B({w}_{k-1},r_k)$.
Recalling that $N$ is the set of noisy examples, and that the
marginals of $P_k$ and $\tilde{P}_k$ on the inputs are the same,
we have
\begin{align*}
& | \E_{(x,y) \in P_k} (\ell(w, x, y)) - \E_{(x,y) \in \tilde{P}_k} (\ell(w, x, y)) | \\
& = | \E_{(x,y) \in \tilde{P}_k} (\ell(w, x, y) - \ell(w,x,\sign(w^* \cdot x))) | \\
& = | \E_{(x,y) \in \tilde{P}_k} (1_{(x,y) \in N} (\ell(w, x, y) - \ell(w,x,-y))) | \\
& \leq \E_{(x,y) \in \tilde{P}_k} (1_{(x,y) \in N} | \ell(w, x, y) - \ell(w,x,-y)|) \\
& \leq 2 \E_{(x,y) \in \tilde{P}_k} \left(1_{(x,y) \in N}
                  \left(\frac{| w \cdot x |}{\tau_k} \right) \right) \\
& = \frac{2}{\tau_k} \E_{(x,y) \in \tilde{P}_k} \left(1_{(x,y) \in N}
                 | w \cdot x | \right) \\
& \leq \frac{2}{\tau_k} \sqrt{\Pr_{(x,y) \sim \tilde{P}_k} (N)}
      \times \sqrt{\E_{(x,y) \in \tilde{P}_k} ((w \cdot x)^2)}
\end{align*}
by the Cauchy-Schwarz inequality.
Lemma~\ref{l:ilc} implies that,
for an absolute constant $c'$,
\[
\Pr_{(x,y) \in \tilde{P}_k} (N)
  \leq \frac{\Pr_{(x,y) \in \tilde{P}} (N)}{\Pr_{(x,y) \in \tilde{P}} (S_{w_{k-1},b_{k-1}})}
  \leq \frac{\eta}{c' M^{-k}}
  \leq \frac{\eta}{c' \epsilon/M}
\]
since $k \leq \lceil \log_M (1/\epsilon) \rceil$, and
Lemma~\ref{lemma:var.small}
implies
$\E_{(x,y) \in \tilde{P}_k} ((w \cdot x)^2) \leq c' z_k^2$.
\end{proof}

Finally, we need some bounds about estimates of the hinge loss.

\begin{lemma}
\label{l:vc.hinge}
Let
\[
\cleaned(W) = \{ (x, \sign(w^* \cdot x)): (x,y) \in W \}.
\]

With
probability $1 - \frac{\delta}{k + k^2}$, for all
$w \in B(w_{k-1},r_k)$, we have
\begin{equation}
\label{e:vc.labelnoise}
| \E_{(x,y) \in \tilde{P}_k}(\ell(w,x,y))
    - \ell(w,W) | \leq \kappa/16
\mbox{, and }
| \E_{(x,y) \in P_k}(\ell(w,x,y))
    - \ell(w,\cleaned(W)) | \leq \kappa/16.
\end{equation}

\end{lemma}
\begin{proof}
See Appendix~\ref{a:vc}.
\end{proof}

\begin{proof}[of Theorem~\ref{thm:low-error-inside-band.labelnoise}]

With probability $1 - \frac{\delta}{k + k^2}$,
we have, for absolute constants $c_1$ and $c_2$, the following:
\begin{eqnarray*}
\err_{D_{w_{k-1}, b_{k-1}}}({w_k})
  & = & \err_{D_{w_{k-1}, b_{k-1}}}({v_k}) \\
  & \leq & \E_{(x,y) \in P_k}(\ell(v_k,x,y))
        \mbox{\hspace{0.1in}(since for each error, the hinge loss is at least $1$)} \\
  & \leq & \E_{(x,y) \in \tilde{P}_k}(\ell(v_k,x,y))
                + c_1 \sqrt{\frac{\eta}{\epsilon}} \times \frac{z_k }{\tau_k}
     \mbox{\hspace{0.1in}(by Lemma~\ref{lemma:clean.vs.dirty.labelnoise})}  \\
  & \leq & \ell({v_k},W)
                + c_1 \sqrt{\frac{\eta}{\epsilon}} \times \frac{z_k }{\tau_k}
            + \kappa/16
            \mbox{\hspace{0.1in}(by Lemma~\ref{l:vc.hinge})}
                                \\
  & \leq & \ell(w^*,W)
                + c_1 \sqrt{\frac{\eta}{\epsilon}} \times \frac{z_k }{\tau_k}
            + \kappa/8
                                \\
  & \leq & \E_{(x,y) \in \tilde{P}_k}(\ell(w^*,x,y))
                + c_1 \sqrt{\frac{\eta}{\epsilon}} \times \frac{z_k  }{\tau_k}
            + \kappa/4
            \mbox{\hspace{0.1in}(by Lemma~\ref{l:vc.hinge})}
                                \\
  & \leq & \E_{(x,y) \in P_k}(\ell(w^*,x,y))
                + c_2 \sqrt{\frac{\eta}{\epsilon}} \times \frac{z_k }{\tau_k}
            + \kappa/4
     \mbox{\hspace{0.1in}(by Lemma~\ref{lemma:clean.vs.dirty.labelnoise})}  \\
  & \leq &  c_2 \sqrt{\frac{\eta}{\epsilon}}  \times \frac{z_k }{\tau_k}
            + \kappa/2,
\end{eqnarray*}
since $L(w^*) \leq \kappa/6$.  Since $z_k/\tau_k = \Theta(1)$, there is
an constant $c_3$ such that,
$\eta \leq c_3 \epsilon$ suffices for
$\err_{D_{w_{k-1}, b_{k-1}}}({w_k}) \leq \kappa$, completing the proof.
\end{proof}

\subsection{Putting it together}
Now we are ready to put everything together.  The proof of
Theorem~\ref{t:adversarial.detailed} follows the high level structure of the
proof of \cite{Balcan07}; the new element is
the application of Theorem~\ref{thm:low-error-inside-band.labelnoise}
which analyzes the performance
of the hinge loss minimization algorithm for learning inside the band.

{\bf Proof} (of Theorem~\ref{t:adversarial.detailed}):
We will prove by induction on $k$ that after $k \leq s$ iterations, we
have $\err_D({w}_k) \leq M^{-k}$
with probability $1-\delta (1-1/(k+1))/2$.

When $k = 0$, all that is required is $\err_D({w}_0) \leq 1$.

Assume now the claim is true for $k-1$ ($k \geq 1$).
Then by induction hypothesis, we
know that with probability at least $1- \delta (1-1/k)/2$,
${w}_{k-1}$ has error at most $M^{-(k-1)}$.
Using Part (e) of Lemma~\ref{l:ilc},
this implies that
$\theta({w}_{k-1},w^*) \leq M^{-(k-1)}/c_6$.
This in turn implies $\theta({w}_{k-1},w^*) \leq \pi/2$.
(When $k = 1$, this is by assumption, and otherwise it is
implied by
Part (e) of Lemma~\ref{l:ilc}.)

Let us define $S_{w_{k-1},b_{k-1}}=\{x: |{w}_{k-1}\cdot
x| \leq b_{k-1}\}$ and $\bar{S}_{w_{k-1},b_{k-1}}=\{x: |{w}_{k-1}\cdot
x| > b_{k-1}\}$.  Since ${w}_{k-1}$ has unit length, and ${v}_k \in
B({w}_{k-1},r_k)$, we have $\theta({w}_{k-1},{v}_k) \leq r_k$
which in turn implies
$\theta({w}_{k-1},{w}_k) \leq \min \{ {M}^{-(k-1)}/c_6, \pi/2 \}$.

Applying
Lemma~\ref{lemma:vectors.ilc} to
bound the error rate outside the band, we have both
\[
\Pr_{x} \left[ ({w}_{k-1} \cdot x ) ({w}_{k} \cdot x)  <
0, x \in \bar{S}_{w_{k-1},b_{k-1}}
 \right]
\leq \frac{{M}^{-k}}{4}
\]
and
\[
\Pr_{x} \left[ ({w}_{k-1} \cdot x ) (w^* \cdot x)  < 0, x \in \bar{S}_{w_{k-1},b_{k-1}}
 \right]
\leq \frac{{M}^{-k}}{4}.
\]

Taking the sum, we obtain $\Pr_{x} \left[ ({w}_{k} \cdot x )
(w^* \cdot x) < 0,
x \in \bar{S}_{w_{k-1},b_{k-1}} \right] \leq \frac{{M}^{-k}}{2}.$ Therefore, we have
\[
\err({w}_k)\leq (\err_{{D}_{w_{k-1},b_{k-1}}}({w}_k))
\Pr({S}_{w_{k-1},b_{k-1}}) + \frac{{M}^{-k}}{2}.
\]
Since $\Pr({S}_{w_{k-1},b_{k-1}}) \leq 2 b_{k-1}$, this implies
\[
\err({w}_k) \leq (\err_{{D}_{w_{k-1},b_{k-1}}}({w}_k))
2 b_{k-1} + \frac{{M}^{-k}}{2}
\leq {M}^{-k}
      \left((\err_{{D}_{w_{k-1},b_{k-1}}}({w}_k)) 2 c_1' M
                  + 1/2\right).
\]

Recall that ${D}_{w_{k-1},b_{k-1}}$ is the distribution obtained by
conditioning $D$ on the event that $x \in S_{w_{k-1},b_{k-1}}$.
Combining Theorem~\ref{thm:low-error-inside-band.labelnoise}
with the induction hypothesis,
\begin{align*}
& \Pr(\err({w}_k) > 1/{M}^k) \\
 & \leq \Pr(\err({w}_k) > 1/{M}^k | \err({w}_{k-1}) \leq 1/{M}^{k-1})
  + \Pr(\err({w}_{k-1}) > 1/{M}^{k-1}) \\
 & \leq \frac{\delta}{2(k+k^2)} + \delta (1-1/k)/2 \\
 & = \delta (1-1/(k+1))/2.
\end{align*}
This completes the proof of the induction, and therefore shows, with
probability at least $1-\delta$,
$O(\log(1/\epsilon))$ iterations suffice to achieve $\err({w}_k) \leq \epsilon$.

A polynomial number of unlabeled samples are required by the algorithm
and the number of labeled examples required by the algorithm is
$\sum_k m_k = O(d (d + \log \log (1/\epsilon)
+ \log(1/\delta)) \log(1/\epsilon))$.
\qed

\noindent \textbf{Remark:} Following the publication of this work, it was brought to our attention by Steve Hanneke that a more careful analysis
of the algorithm achieves the right labeled sample complexity of $O(d + \log \log (1/\epsilon) + \log(1/\delta))\log(1/\epsilon)$. This is implicit in~\cite{HKL15}\footnote{See \href{https://39afa9d6-a-62cb3a1a-s-sites.googlegroups.com/site/stevehanneke/docs/2015/agnostic-linsep-uniform-note.pdf?attachauth=ANoY7cpcOAMqDmplO1NbVDJwYvWWSO7lrK4rgkBcB7BchHlmZzISs8J9hlRI-oy9qaipHkZxjxfeWauS2EwKALCSipwz_wJ9Fa82FaqKNcvvQYyi06wGdLCaqvp0yNgMGhSzAcE02XQuJpgo4XKr0l2Mvv_vmkPajkAX_24C39QFXnpzLAKirD--48twaMdK8He4Xa_bwE6pjuxKSjmvKunF1effwp-ghTDzbvLgOPdBwg8iDohEHO1enuhPqvAfet-Mat8TO6fT&attredirects=0}{[link]} for a self contained proof graciously communicated to us by Steve Hanneke.}

\section{Learning with malicious noise}
\label{s:malicious}

The intuition in the case of malicious noise is the same as
for adversarial noise, except that, because the adversary can also change
the marginal distribution over the instances, it is necessary to
perform an additional outlier removal step at each stage of the algorithm.
Furthermore, we need a different analysis since in this case the marginal distribution over the examples
can change.

\begin{fragment*}[t]
\caption{\label{fig:active-algorithm-malicious}{\sc Computationally Efficient Algorithm tolerating malicious noise}
}
\textbf{Input}: allowed error rate $\epsilon$, probability of failure $\delta$,
an oracle that returns $x$, for $(x,y)$ sampled from $\mathrm{EX}_{\eta} (f,D)$,
and an oracle for getting the label $y$ from an example;
a sequence of unlabeled sample sizes $n_k > 0$, $k \in \integers$;
a sequence of labeled sample sizes $m_k>0$;
a sequence of cut-off values $b_k >0$;
a sequence of hypothesis space radii $r_k >0$;
a sequence of removal rates $\xi_k$;
a sequence of variance bounds $\sigma^2_k$; precision value $\kappa$; weight vector $w_0$.
 \vspace*{0.01in}
\begin{enumerate} \itemsep 0pt
\small
\item Draw $n_1$ unlabeled examples and put them into a working set $W$.
\item For $k=1,\ldots, \rounds = \lceil \log_2(1/\epsilon) \rceil$
\begin{enumerate}
\item Apply Algorithm~\ref{fig:outlier-removal} to $W$ with parameters
$u \leftarrow {w}_{k-1}$, $\gamma \leftarrow b_{k-1}$, $r \leftarrow r_k$, $\xi \leftarrow \xi_k$, $\sigma^2 \leftarrow \sigma^2_k$ and let $q$ be the output function $q: W \rightarrow [0,1]$ . Normalize $q$ to form a probability
   distribution $p$ over $W$.
\item Choose $m_k$ examples from $W$ according to $p$ and reveal their labels. Call this set $T$.
\item Find ${v}_k \in B({w}_{k-1},r_k)$
to approximately minimize training hinge loss over $T$
s.t. $\Vert {v}_k \Vert_2 \leq 1$: \\
$\ell_{\tau_k}({v}_k ,T) \leq \min_{w \in B({w}_{k-1},r_k) \cap B(0,1))}
            \ell_{\tau_k}(w,T) + \kappa/8$.  \- \-\\
 Normalize ${v}_k$ to have unit length, yielding
${w}_k = \frac{{v}_k}{\Vert {v}_k \Vert_2}$.
\item Clear the working set $W$.
\item \textbf{Until} $n_{k+1}$ additional data points are put in $W$, given
 unlabeled $x$ for $(x,f(x))$ obtained from $\mathrm{EX}_{\eta} (f,D)$,
  \textbf{if} $|{w}_{k}\cdot x| \geq b_k$, \textbf{then} reject $x$
 \textbf{else} put into $W$.
\end{enumerate}
\end{enumerate}
\textbf{Output}: weight vector $w_{\rounds}$ of error at
most $\epsilon$ with probability $1-\delta$.
\end{fragment*}

\begin{fragment*}[t]
\caption{\label{fig:outlier-removal}{\sc Localized soft outlier removal procedure}
} \vspace*{0.01in}

 \textbf{Input}: a set $S = \{(x_1, x_2, \ldots, x_n)\}$ of samples;
                the reference unit vector $u$; 
                desired radius $r$;
                a parameter $\xi$ specifying the desired bound on the fraction of clean examples removed;
                a variance bound $\sigma^2$
\begin{enumerate} \itemsep 0pt
\small
\item Find $q: S \rightarrow [0,1]$ satisfying the following constraints:
\begin{enumerate}
\item for all $x \in S$, $0 \leq q(x) \leq 1$
\item $\frac{1}{|S|} \sum_{(x,y) \in S} q(x) \geq 1 - \xi$
\item for all
$w \in B(u, r) \cap B({\bf 0},1),
\; \frac{1}{|S|} \sum_{x \in S} q(x) (w\cdot x)^2
   \leq \sigma^2$.
\end{enumerate}
\end{enumerate}
\textbf{Output}: A function $q: S \rightarrow [0,1]$.
\end{fragment*}


Theorem~\ref{thm:log-concave-malicious} follows immediately from the
following theorem analyzing Algorithm~\ref{fig:active-algorithm-malicious}.
\begin{theorem}
\label{t:malicious.detailed}
Let a distribution $D$ over $\reals^d$ be isotropic log-concave.  Let
$w^*$ be the (unit length) target weight vector.  There are settings
of the parameters of Algorithm~\ref{fig:active-algorithm-malicious},
and positive constants $M$, $C$ and $\epsilon_0$, such that, for all
$\epsilon < \epsilon_0$, for any $\delta>0$, if the rate
$\eta$ of malicious noise satisfies $\eta < C \epsilon$, a number $n_k
= \mathrm{poly}(d,M^k, \log(1/\delta))$ of unlabeled examples in round
$k$ and a number $m_k = O\left(d
\log\left(\frac{d}{\epsilon\delta}\right) (d + \log(k/ \delta)
)\right)$ of labeled examples in round $k \geq 1$, and $w_0$ such that
$\theta(w_0,w^*) < \pi/2$, after $s=O(\log(1/\epsilon))$ iterations,
finds $w_{\rounds}$ satisfying $\err(w_{\rounds}) \leq \epsilon$ with
probability $\geq 1-\delta$.
%
\end{theorem}

The rest of this section is dedicated to the proof of
Theorem~\ref{t:malicious.detailed}.

\subsection{Parameters for the algorithm}

With the exception of the parameters $\sigma^2_k$ and $\xi_k$
of the outlier removal procedure,
the parameters are set exactly as in Section~\ref{s:parameters.labelnoise}.

The values of $\sigma^2_k$ and $\xi_k$ are
determined by our analysis: $\sigma^2_k$ is
$c (r^2_k + b^2_{k-1})$, for the value of
$c$ in Theorem~\ref{thm:lopchop} below,
that corresponds to the choice, in the statement of
Theorem~\ref{thm:lopchop}, of $C = c_1'$.
Finally,
$\xi_k = \min(\frac{\kappa}{2^7},
              \frac{\kappa^2\tau^2_k}{c_4 2^{16} z^2_k })$,
for the value of $c_4$ in Lemma~\ref{lemma:var.small} corresponding
to the choice $c_3 = b_0$.

\subsection{Analysis of the outlier removal subroutine}

The analysis of the learning algorithm uses the following
lemma about Algorithm~\ref{fig:outlier-removal}.

\begin{theorem}
\label{thm:lopchop}
For any $C > 0$, there is a constant $c$ and a polynomial $p$ such that,
for all $\xi > 2 \eta'$ and all $0 < \gamma < C$,
if $n \geq
p(1/\eta', d, 1/\xi, 1/\delta,1/\gamma,1/r)$, then,
with probability $1 - \delta$, the output $q$ of Algorithm~\ref{fig:outlier-removal} satisfies the following:
\begin{itemize}
\item $\sum_{x \in S} q(x) \geq (1 - \xi) |S| $ (a fraction $1 - \xi$ of
the weight is retained)
\item For all unit length $w$ such that $\Vert w-u \Vert_2 \le r$,
\begin{equation}
\label{e:lopchop.varbound}
\frac{1}{|S|} \sum_{x \in S} q(x) (w \cdot x)^2
   \leq c (r^2 + \gamma^2).
\end{equation}
\end{itemize}
Furthermore, the algorithm can be implemented in polynomial time.
\end{theorem}

Our proof of Theorem~\ref{thm:lopchop} proceeds through a series of
lemmas.
Lemma~\ref{l:ilc} implies that we
may assume without loss of generality that the instances $x_1,...,x_n$
from $S$ are distinct.
Obviously, a feasible $q$ satisfies the
requirements of the lemma.  So all we need to show is
\begin{itemize}
\item there is a feasible solution $q$, and
\item we can simulate a separation oracle: given a provisional solution
$\hat{q}$, we can find a linear constraint violated by $\hat{q}$
in polynomial time.
\end{itemize}

We will start by working on proving that there is a feasible
$q$.  First of all, a Chernoff bound implies that
$n \geq \mathrm{poly}(1/\eta', 1/\delta)$ suffices for it to
be the case that, with probability
$1 - \delta$, at most $2 \eta'$ members of $S$ are noisy.  Let
us assume from now on that this is the case.

We will show that $q^*$ which sets
$q^*(x) = 0$ for each noisy point, and $q^*(x) = 1$
for each non-noisy point, is feasible.

First, we use VC tools to show that,
if enough examples are chosen, a bound like
Lemma~\ref{lemma:var.small},
but averaged over the clean examples, likely holds
for all relevant directions.
\begin{lemma}
\label{lemma:all.var}
If we draw $\ell$ times i.i.d. from
$D$ to form $X_C$, with probability $1 - \delta$, we have that for
any unit length $a$,
\[
\frac{1}{\ell} \sum_{x \in X_C} (a \cdot x)^2
   \leq E[(a\cdot x)^2]
       +  \sqrt{\frac{O(d \log(\ell/\delta) (d + \log (1/\delta)))}{\ell}}.
\]
\end{lemma}
{\bf Proof:}  See Appendix~\ref{a:vc}.
\qed

Lemma~\ref{lemma:all.var} and
Lemma~\ref{lemma:var.small} together
directly imply that
\[
n = \mathrm{poly}\left(d, 1/\eta', 1/\delta,\frac{1}{c (r^2 + \gamma^2)}\right)
 = \mathrm{poly}\left(d, 1/\eta', 1/\delta,1/\gamma,1/r\right)
\]
suffices for it to be the case that, for all
$w \in B(u,r)$,
\[
\frac{1}{|S|} \sum_{(x,y)} q^*(x) (a \cdot x)^2
  \leq 2 \E[(a \cdot x)^2]
 \leq 2 c_4 (r^2 + \gamma^2),
\]
where $c_4$ is the value in Lemma~\ref{lemma:var.small}
corresponding
to setting $c_3 = C$.  If $c = 2 c_4$, we have that $q^*$ is feasible.

So what is left is to prove is that a separation oracle
for the convex program can be computed in polynomial time.
Very roughly, there is a linear constraint for each of a set of directions,
limiting the variance in that direction.   We can find a violated
constraint, if there is one, by finding the direction with maximum
variance, using something like PCA, but taking appropriate account of the
fact that we are only considering directions near $u$.

In detail, we may compute the separation oracle as follows.
First, it is easy to check whether,
for all $x$,
$0 \leq q(x) \leq 1$, and whether
$\sum_{x \in S} q(x) \geq (1 - \xi) |S|$.
An algorithm can first do that.  If these pass, then it needs to
check whether there is a $w \in B(u,r)$ with $|| w ||_2 \leq 1$ such that
\[
\frac{1}{|S|} \sum_{x \in S} q(x) (w\cdot x)^2
   > c (r^2 + \gamma^2).
\]
This can be done by finding $w \in B(u,r)$ with $|| w ||_2 \leq 1$ that
maximizes $\sum_{x \in S} q(x) (w\cdot x)^2$, and checking it.

Suppose $X$ is a matrix with a row for each $x \in S$, where the
row is $\sqrt{q(x)} x$.
Then $\sum_{x \in S} q(x) (w\cdot x)^2 = w^T X^T X w$,
and, maximizing this over $w$ is an equivalent problem
to minimizing $w^T (-X^T X) w$ subject to
$\Vert w-u \Vert_2 \le r$ and $|| w || \leq 1$.  Since
$-X^T X$ is symmetric, problems of this form are known to be
solvable in polynomial time \cite{SZ03} (see \cite{BM14}).

\subsection{The error within a band in each iteration}
At each iteration, Algorithm~\ref{fig:active-algorithm-malicious}
concentrates its attention on examples in the band.  Our next theorem
analyzes its error on these examples.

\begin{theorem}
\label{thm:low-error-inside-band}
For $k \leq \lceil \log_M (1/\epsilon) \rceil$,
if $err_{D}(w_{k-1}) \le M^{-(k-1)}$,
with probability $1 - \frac{\delta}{k+k^2}$ (over the random examples
in round $k$),
after round $k$ of Algorithm~\ref{fig:active-algorithm-malicious},
we have $err_{D_{w_{k-1}, b_{k-1}}}(w_k) \le \kappa$.
\end{theorem}

We will prove Theorem~\ref{thm:low-error-inside-band} using a series
of lemmas below. First, we bound the hinge loss of the target $w^*$
within the band $S_{w_{k-1}, b_{k-1}}$.  Since we are analyzing a
particular round $k$, to reduce clutter in the formulas, for the rest
of this section, let us refer to
$\ell_{\tau_k}$ as $\ell$, and
$L_{\tau_k}(\cdot,D_{w_{k-1}, b_{k-1}})$ as $L(\cdot)$. First, Lemma~\ref{lemma:opt.hinge.underlying},
that $L(w^*) \leq \kappa/6$, also applies here, using exactly
the same proof.

%
%

During round $k$ we can decompose the working set $W$ into the set of
``clean'' examples $W_C$ which are drawn from $D_{w_{k-1}, b_{k-1}}$
and the set of ``dirty'' or malicious examples $W_D$ which are output
by the adversary. We will next show that the fraction of dirty
examples in round $k$ is not too large.

\begin{lemma}
\label{lemma:few.noisy}
There is an absolute positive constant $c$ such that, with
probability $1 - \frac{\delta}{6 (k+k^2)}$,
\begin{equation}
\label{e:few.noisy}
|W_D| \leq c \eta n_k {M}^k \leq \frac{ c M \eta n_k}{\epsilon}
\end{equation}
\end{lemma}
\begin{proof}
From Lemma~\ref{l:ilc} and the setting of our parameters, the
probability that an example falls in $S_{w_{k-1}}$ is at least
$\Omega(M^{-k})$. Therefore, with probability $(1 - \frac{\delta}{12
  (k+k^2)})$, the number of examples we must draw before we encounter
$n_k$ examples that fall within $S_{w_{k-1}, b_{k-1}}$ is at most
$O({n_k {M}^k})$.  The probability that each unlabeled example we draw
is noisy is at most $\eta$. Applying a Chernoff bound, with
probability at least $1 - \frac{\delta}{12 (k+k^2)}$, we have
$|W_D| \leq c \eta n_k M^k$.  Since $k \leq \lceil \log_M (1/\epsilon) \rceil$,
this completes the proof.
\end{proof}

Recall that the total variation distance between two probability
distributions is the maximum difference between the probabilities
that they assign to any event.

We can think of $q$ as soft indicator functions for ``keeping''
examples, and so interpret the inequality
$\sum_{x \in W} q(x) \geq (1 - \xi)|W|$ as roughly akin to saying
that most examples are kept.  This means that distribution
$p$ obtained by normalizing $q$ is close to the uniform
distribution over $W$.  We make this precise in the following lemma.
\begin{lemma}
\label{l:almost.uniform}
The total variation distance between $p$ and the uniform distribution
over $W$ is at most $\xi$.
\end{lemma}
\begin{proof}
Lemma~1 of \cite{LS06} implies that the total variation distance
$\rho$ between $p$ and the uniform distribution over $W$ satisfies
\[
\rho
= 1 - \sum_{x \in W} \min \left\{ p(x), \frac{1}{|W|} \right\}
= 1 - \sum_{x \in W} \min \left\{ \frac{q(x)}{\sum_{u \in W} q(u)},
                                    \frac{1}{|W|} \right\}.
\]
Since $q(u) \leq 1$ for all $u$, we have
$\sum_{u \in W} q(u) \leq |W|$, so that
\[
\rho \leq 1 - \frac{1}{|W|} \sum_{x \in W} \min \{ q (x),  1 \}.
\]
Again, since $q(x) \leq 1$, we have
\[
\rho \leq 1 - \frac{(1 - \xi) |W|}{|W|}  = \xi.
\]
\end{proof}

Next, we will relate the average hinge loss when examples
are weighted according to $p$, i.e., $\ell(w,p)$ to the
hinge loss averaged over clean examples $W_C$, i.e., $\ell(w,W_C)$.
Here $\ell(w,W_C)$ and $\ell(w,p)$ are defined with respect to the
unrevealed labels that the adversary has committed to.
\begin{lemma}
\label{lemma:clean.vs.filtered.lc}
There are absolute constants $C_1$, $C_2$ and $C_3$ such that,
with probability $1 - \frac{\delta}{2 (k+k^2)}$,
if we define $z_k = \sqrt{r_k^2 + {b^2_{k-1}}}$, then
for any $w \in B({w}_{k-1},r_k)$, we have
\begin{equation}
\label{e:clean.by.filtered}
\ell(w, W_C)
\leq \ell(w,p)
        + \frac{C_1 \eta}{\epsilon}
    \left(1 + \frac{z_k}{\tau_k}
                        \right) + \kappa/32
\end{equation}
and
\begin{equation}
\label{e:filtered.by.clean}
\ell(w,p)
  \leq 2 \ell(w,W_C)
        + \kappa/32
        +  \frac{C_2 \eta}{\epsilon}
        + C_3 \sqrt{\frac{\eta}{\epsilon}} \times \frac{z_k}{\tau_k}.
\end{equation}
\end{lemma}
\begin{proof}
Assume without loss of generality that each element $(x,y) \in W$ is distinct. Fix an arbitrary $w \in B({w}_{k-1},r_k)$.
By
Theorem~\ref{thm:lopchop},
Lemma~\ref{lemma:few.noisy},
Lemma~\ref{l:ilc},
Lemma~\ref{lemma:var.small}, and
Lemma~\ref{lemma:all.var}, we know that, with probability
$1 - \frac{\delta}{2(k+k^2)}$, there are absolute constants
$K_1$, $K_2$ and $K_3$ such that
\begin{eqnarray}
\label{e:tvar.q}
\frac{1}{|W|}
  \sum_{x \in W} q(x) (w \cdot x)^2 & \leq & K_1 z^2_k \\
\label{e:few.noisy.lc}
|W_D| & \leq & \frac{K_2 \eta n_k}{\epsilon} \\
\label{e:clean.var}
\frac{1}{|W_C|} \sum_{(x,y) \in W_C} (w \cdot x)^2
     & \leq & K_3 z_k^2.
\end{eqnarray}
(We will need the value of $K_3$ later: we may use
\begin{equation}
\label{e:K_3}
K_3 = 2 c_4
\end{equation}
for the value of $c_4$ in Lemma~\ref{lemma:var.small} corresponding
to $c_3 = b_0$.)

Assume that (\ref{e:tvar.q}), (\ref{e:few.noisy.lc}) and
(\ref{e:clean.var}) all hold.

Since $\sum_{x \in W} q(x) \geq (1 - \xi_k) |W| \geq |W|/2$,
we have that (\ref{e:tvar.q}) implies
\begin{equation}
\label{e:tvar}
\sum_{x \in W} p(x) (w \cdot x)^2 \leq {2 K_1}z^2_k.
\end{equation}

First, let us bound the weighted loss on noisy examples in the training
set. In particular, we will show that
 \begin{equation}
 \label{e:noisy.loss}
  \sum_{(x,y) \in W_D} p(x) \ell(w,x,y) \le
    K_2 \eta/\epsilon  + \xi_k
      + \sqrt{2 K_1 K_2 \eta/\epsilon + \xi_k} \left(\frac{z_k}{\tau_k}\right).
 \end{equation}

To see this, notice that,
\begin{align*}
& \sum_{(x,y) \in W_D} p(x) \ell(w,x,y)
= \sum_{(x,y) \in W_D} p(x) \max \left\{ 0, 1 - \frac{y (w \cdot x)}{\tau_k} \right\} \\
&  \leq \Pr_p(W_D)
      + \frac{1}{\tau_k} \sum_{(x,y) \in W_D} p(x) |w \cdot x|
  = \Pr_p(W_D)
      + \frac{1}{\tau_k} \sum_{(x,y) \in W} p(x) 1_{W_D}(x,y) |w \cdot x| \\
&  \leq \Pr_p(W_D)
      + \frac{1}{\tau_k} \sqrt{ \sum_{(x,y) \in W} p(x) 1_{W_D}(x,y)}
                         \sqrt{ \sum_{(x,y) \in W} p(x) (w \cdot x)^2}
  \;\;\mbox{(by the Cauchy-Schwarz inequality)}
\\
&  \leq \Pr_p(W_D)
      + \sqrt{2 K_1 \Pr_p(W_D)} \left( \frac{ z_k}{\tau_k} \right)
 \leq \frac{K_2 \eta}{\epsilon} + \xi_k
      + \sqrt{2 K_1 K_2 \eta/\epsilon + \xi_k} \left( \frac{ z_k}{\tau_k} \right)
 \end{align*}
where the second to last inequality follows by (\ref{e:tvar}) and the last one by Lemma~\ref{l:almost.uniform} and (\ref{e:few.noisy.lc}).

Similarly, we will show that
\begin{equation}
\label{e:total.loss}
\sum_{(x,y) \in W} p(x) \ell(w,x,y) \le 1
   + \sqrt{2 K_1} \left( \frac{ z_k}{\tau_k} \right).
\end{equation}
To see this notice that,
\begin{align*}
& \sum_{(x,y) \in W} p(x) \ell(w,x,y)
 = \sum_{(x,y) \in W} p(x) \max \left\{ 0,
                             1 - \frac{y (w \cdot x)}{\tau_k} \right\} \\
& \leq 1
      + \frac{1}{\tau_k} \sum_{(x,y) \in W} p(x) |w \cdot x|
 \leq 1
      + \frac{1}{\tau_k}
       \sqrt{\sum_{(x,y) \in W} p(x) (w \cdot x)^2} \\
& \leq 1
 + \sqrt{2 K_1} \left( \frac{ z_k}{\tau_k} \right),
 \end{align*}
by (\ref{e:tvar}).

Next, we have
\begin{align*}
\ell(w,W_C)
 & = \frac{1}{|W_C|} \left( \sum_{(x,y) \in W} q(x) \ell(w,x,y)
           + (1_{W_C} (x,y) - q(x)) \ell(w,x,y) \right) \\
 & \leq \frac{1}{|W_C|} \left( \sum_{(x,y) \in W} q(x) \ell(w,x,y)
           + \sum_{(x,y) \in W_C} (1 - q(x)) \ell(w,x,y) \right) \\
 & \leq \frac{1}{|W_C|} \left( \sum_{(x,y) \in W} q(x) \ell(w,x,y)
  + \sum_{(x,y) \in W_C} (1 - q(x))
                        \left(1 + \frac{|w \cdot x|}{\tau_k} \right)
                \right) \\
 & \leq \frac{1}{|W_C|} \left( \sum_{(x,y) \in W} q(x) \ell(w,x,y)
  + \xi_k |W|
  + \frac{1}{\tau_k} \sum_{(x,y) \in W_C} (1 - q(x)) |w \cdot x| \right) \\
 & \leq \frac{1}{|W_C|} \left( \sum_{(x,y) \in W} q(x) \ell(w,x,y)
  + \xi_k |W|
  + \frac{1}{\tau_k} \sqrt{ \sum_{(x,y) \in W_C} (1 - q(x))^2}
                     \sqrt{ \sum_{(x,y) \in W_C} (w \cdot x)^2 } \right)
\end{align*}
by the Cauchy-Schwarz inequality.  Recall that $0 \leq q(x) \leq 1$,
and $\sum_{(x,y) \in W} q(x) \geq 1 - \xi_k |W|$.  Thus,
\begin{align*}
\ell(w,W_C)
 & \leq \frac{1}{|W_C|} \left( \sum_{(x,y) \in W} q(x) \ell(w,x,y)
  + \xi_k |W|
  + \frac{1}{\tau_k} \sqrt{ \xi_k |W| }
                     \sqrt{ \sum_{(x,y) \in W_C} (w \cdot x)^2 }\right) \\
 & \leq \frac{1}{|W_C|} \left( \sum_{(x,y) \in W} q(x) \ell(w,x,y)
  + \xi_k |W|
  + \sqrt{ \xi_k |W| |W_C| K_3 } \left( \frac{ z_k}
         {\tau_k} \right) \right)
\end{align*}
by (\ref{e:clean.var}).  Since $|W_C| \geq |W|/2$, we have
\begin{align*}
\ell(w,W_C)
 & \leq \frac{1}{|W_C|} \left( \sum_{(x,y) \in W} q(x) \ell(w,x,y) \right)
  + 2 \xi_k + \sqrt{ 2 \xi_k K_3 } \left( \frac{z_k}{\tau_k} \right). \\
\end{align*}
We have chosen $\xi_k$ small enough that
\begin{align*}
\ell(w,W_C)
 & \leq \frac{1}{|W_C|} \left( \sum_{(x,y) \in W} q(x) \ell(w,x,y) \right)
         + \kappa/32 \\
 & = \frac{\sum_{(x,y) \in W} q(x)}{|W_C|}
          \left( \sum_{(x,y) \in W} p(x) \ell(w,x,y) \right)
         + \kappa/32 \\
 & = \ell(w,p) +
  \left( \frac{\sum_{(x,y) \in W} q(x)}{|W_C|} - 1 \right)
                    \left( \sum_{(x,y) \in W} p(x) \ell(w,x,y) \right)
         + \kappa/32 \\
 & \leq \ell(w,p) +
  \left( \frac{|W|}{|W_C|} - 1 \right)
           \left( \sum_{(x,y) \in W} p(x) \ell(w,x,y) \right)
         + \kappa/32 \\
 & \leq \ell(w,p) +
  \left( \frac{|W|}{|W_C|} - 1 \right)
   \left(1 + \sqrt{2 K_1} \left( \frac{ z_k}{\tau_k} \right) \right)
         + \kappa/32, \\
\end{align*}
by (\ref{e:total.loss}).
Applying (\ref{e:few.noisy.lc}) yields (\ref{e:clean.by.filtered}).

Also,
\begin{align*}
\ell(w,p) & = \sum_{(x,y) \in W} p(x) \ell(w,x,y) \\
& = \sum_{(x,y) \in W_C} p(x) \ell(w,x,y)
     + \sum_{(x,y) \in W_D} p(x) \ell(w,x,y) \\
& \leq \sum_{(x,y) \in W_C} p(x) \ell(w,x,y)
       +  K_2 \eta / \epsilon  + \xi_k
      + \sqrt{2 K_1 K_2 \eta / \epsilon + \xi_k} \left( \frac{ z_k}{\tau_k} \right)
         \mbox{\hspace{0.1in} (by (\ref{e:noisy.loss}))}. \\
& = \frac{\sum_{(x,y) \in W_C} q(x) \ell(w,x,y)}{\sum_{(x,y) \in W_C} q(x)}
       +  K_2 \eta / \epsilon  + \xi_k
      + \sqrt{2 K_1 K_2 \eta / \epsilon + \xi_k} \left( \frac{ z_k}{\tau_k} \right)
    \\
& \leq \frac{\sum_{(x,y) \in W_C} \ell(w,x,y)}{\sum_{(x,y) \in W_C} q(x)}
       +  K_2 \eta / \epsilon  + \xi_k
      + \sqrt{2 K_1 K_2 \eta / \epsilon + \xi_k} \left( \frac{ z_k}{\tau_k} \right)
   \mbox{\hspace{0.1in} (since $\forall x, q(x) \leq 1$))}. \\
& \leq \frac{\sum_{(x,y) \in W_C} \ell(w,x,y)}{|W_C| - \xi |W|}
       +  K_2 \eta / \epsilon  + \xi_k
      + \sqrt{2 K_1 K_2 \eta / \epsilon + \xi_k}
              \left( \frac{ z_k}{\tau_k} \right) \\
& \leq 2 \ell(w,W_C)
       +  K_2 \eta / \epsilon  + \xi_k
      + \sqrt{2 K_1 K_2 \eta / \epsilon + \xi_k} \left( \frac{ z_k}{\tau_k} \right),
\end{align*}
by (\ref{e:few.noisy.lc}), which in turn implies (\ref{e:filtered.by.clean}).
\end{proof}

\begin{proof}[of Theorem~\ref{thm:low-error-inside-band}]
Exploiting the fact that, with high probability,
$\ell(w,x,y) = O\left(\sqrt{d \log \left( \frac{d}{\epsilon \delta} \right)}\right)$ for
all $(x,y) \in S_{w_{k-1},b_{k-1}}$ and $w \in B(w_{k-1},r_k)$
as in the proof of Lemma~\ref{l:vc.hinge},
with probability $1 - \frac{\delta}{2 (k + k^2)}$,
for all $w \in B(w_{k-1}, r_k)$,
\begin{equation}
\label{e:W_C.est}
|L(w) - \ell(w,W_C)| \leq \kappa/32
\end{equation}
and
\begin{equation}
\label{e:T.est}
|\ell(w,p) - \ell(w,T)| \leq \kappa/32.
\end{equation}
Also with probability $1-\frac{\delta}{2(k+k^2)}$, both
(\ref{e:clean.by.filtered}) and (\ref{e:filtered.by.clean}) hold.
Let us assume from here on that all of these hold.

Then we have
\begin{eqnarray*}
\err_{D_{w_{k-1}, b_{k-1}}}({w_k}) & = & \err_{D_{w_{k-1}, b_{k-1}}}({v_k}) \\
  & \leq & L({v_k})
        \mbox{\hspace{0.1in}(since for each error, the hinge loss is at least $1$)} \\
  & \leq & \ell({v_k},W_C) + \kappa/16
        \mbox{\hspace{0.1in}(by (\ref{e:W_C.est}))} \\
  & \leq & \ell({v_k},p)
   + \frac{C_1 \eta}{\epsilon} \left(1 + \frac{z_k}{\tau_k}
                        \right)
   + \kappa/8
        \mbox{\hspace{0.1in}(by (\ref{e:clean.by.filtered}))} \\
  & \leq & \ell({v_k},T)
   + \frac{C_1 \eta}{\epsilon} \left(1 + \frac{z_k}{\tau_k}
                        \right)
   + \kappa/4
        \mbox{\hspace{0.1in}(by (\ref{e:T.est}))} \\
  & \leq & \ell(w^*,T)
   + \frac{C_1 \eta}{\epsilon} \left(1 + \frac{z_k}{\tau_k}
                        \right)
   + \kappa/4
    \mbox{\hspace{0.1in}(since $w^* \in B({w}_{k-1},r_k)$)} \\
   & \leq & \ell(w^*,p)
   + \frac{C_1 \eta}{\epsilon} \left(1 + \frac{z_k}{\tau_k}
                        \right)
   + \kappa/3
    \mbox{\hspace{0.1in}(by (\ref{e:T.est}))}.
\end{eqnarray*}

This, together with (\ref{e:filtered.by.clean}) and (\ref{e:W_C.est}), gives
\begin{eqnarray*}
\err_{D_{w_{k-1}, b_{k-1}}}({w_k})
  & \leq & 2 \ell(w^*,W_C)
        +  \frac{C_2 \eta}{\epsilon}
        + C_3 \sqrt{\frac{\eta}{\epsilon}} \times \frac{ z_k}{\tau_k} \\
  & & \hspace{0.5in}
     + \frac{C_1 \eta}{\epsilon} \left(1 + \frac{z_k}
                                           {\tau_k} \right)
   + 2 \kappa/5 \\
  & \leq & 2 L(w^*)
        +  \frac{C_2 \eta}{\epsilon}
        + C_3 \sqrt{\frac{\eta}{\epsilon}} \times \frac{ z_k}{\tau_k}
   + \frac{C_1 \eta}{\epsilon} \left(1 + \frac{z_k}
                                           {\tau_k} \right)
   + \kappa/2 \\
  & \leq & \kappa/3
        +  \frac{C_2 \eta}{\epsilon}
        + C_3 \sqrt{\frac{\eta}{\epsilon}} \times \frac{ z_k}{\tau_k}
   + \frac{C_1 \eta}{\epsilon} \left(1 + \frac{z_k}{\tau_k} \right)
   + \kappa/2,
\end{eqnarray*}
by Lemma~\ref{lemma:opt.hinge.underlying}.

Now notice that $z_k/\tau_k$ is $\Theta(1)$.
Hence an $\Omega(\epsilon)$ bound on
$\eta$ suffices to imply
that $\err_{D_{w_{k-1}, b_{k-1}}}({w_k}) \le \kappa$ with probability
$(1-\frac{\delta}{k+k^2})$.
\end{proof}
The rest of the analysis is exactly the same as for
the case of adversarial label noise.

\section{Discussion}
We note that the idea of localization in the concept
space is traditionally used in statistical learning theory both in
supervised and active learning for getting sharper
rates~\cite{bbl05,BLL09,Kol10}. Furthermore, the idea of localization
in the instance space has been used in margin-based analysis of active
learning~\cite{Balcan07,BL13}. In this work we used localization in
both senses in order to get polynomial-time algorithms with better
noise tolerance. It would be interesting to further exploit this idea
for other concept spaces.

Our algorithms run in polynomial time, and therefore use a polynomial
number of examples.  Notably, they use only polylogarithmically many
class labels.  Our bounds on the total number of examples used by our
algorithms are, however, somewhat worse than the best bounds known for
the noise-free case.  In order to find and remove outliers, the
precision with which we need statistics on the training data to match
properties of the underlying distribution gets finer as the number of
variables increases.  When combined with the usual effect in VC
analyses regarding growth of the richness of behavior with the number
of variables (which could be partially mitigated using localized
analysis in place of the VC tools that we have used here), this leads
to the increased requirement on the number of examples.  Substantially
improving the sample complexity and finding more computationally
efficient noise-tolerant algorithms is a potentially useful topic for
future research.

While we have chosen to focus on isotropic log-concave distributions
to present our techniques in a clean setting, it appears that, using
tools from \cite{BL13,ABL14}, our analysis can be applied to a broader
class of distributions with minor changes, including ``nearly
log-concave distributions'', defined as in \cite{AK91}.  One property
of the distribution that is needed for our analysis is that it is
fairly likely that a random example falls fairly close to the
separating hyperplane of the target.  While this may not be the case in some
applications, such applications are typically easier, and might be
handled separately.  Provably noise-tolerant learning of linear
classifiers for natural classes of distributions that include such
cases is another important topic for future work.

\subsection*{Acknowledgments} We thank Steve Hanneke for helpful communications.  We also thank
anonymous reviewers for their helpful comments.
This work was supported in part by NSF grants CCF-0953192,
CCF-1101283, and CCF-1422910, AFOSR grant FA9550-09-1-0538, ONR grant N00014-09-1-0751,
and a Microsoft Research Faculty Fellowship.

\bibliographystyle{ACM-Reference-Format-Journals}
\bibliography{paper-al}


\begin{thebibliography}{00}


\ifx \showCODEN    \undefined \def \showCODEN     #1{\unskip}     \fi
\ifx \showDOI      \undefined \def \showDOI       #1{{\tt DOI:}\penalty0{#1}\ }
  \fi
\ifx \showISBNx    \undefined \def \showISBNx     #1{\unskip}     \fi
\ifx \showISBNxiii \undefined \def \showISBNxiii  #1{\unskip}     \fi
\ifx \showISSN     \undefined \def \showISSN      #1{\unskip}     \fi
\ifx \showLCCN     \undefined \def \showLCCN      #1{\unskip}     \fi
\ifx \shownote     \undefined \def \shownote      #1{#1}          \fi
\ifx \showarticletitle \undefined \def \showarticletitle #1{#1}   \fi
\ifx \showURL      \undefined \def \showURL       #1{#1}          \fi

\bibitem[\protect\citeauthoryear{Anthony and Bartlett}{Anthony and
  Bartlett}{1999}]%
        {AB99}
{M. Anthony} {and} {P.~L. Bartlett}. 1999.
\newblock {\em Neural Network Learning: Theoretical Foundations}.
\newblock Cambridge University Press.
\newblock


\bibitem[\protect\citeauthoryear{Applegate and Kannan}{Applegate and
  Kannan}{1991}]%
        {AK91}
{D. Applegate} {and} {R. Kannan}. 1991.
\newblock \showarticletitle{Sampling and integration of near log-concave
  functions}. In {\em STOC}.
\newblock


\bibitem[\protect\citeauthoryear{Arora, Babai, Stern, and Sweedyk}{Arora
  et~al\mbox{.}}{1993}]%
        {ABSS97}
{S. Arora}, {L. Babai}, {J. Stern}, {and} {Z. Sweedyk}. 1993.
\newblock \showarticletitle{The hardness of approximate optima in lattices,
  codes, and systems of linear equations}. In {\em Proceedings of the 1993 IEEE
  34th Annual Foundations of Computer Science}.
\newblock


\bibitem[\protect\citeauthoryear{Awasthi, Balcan, and Long}{Awasthi
  et~al\mbox{.}}{2014}]%
        {ABL14}
{P. Awasthi}, {M. Balcan}, {and} {P.~M. Long}. 2014.
\newblock \showarticletitle{The power of localization for efficiently learning
  linear separators with noise}. In {\em STOC}. 449--458.
\newblock
\newblock
\shownote{See also Arxiv paper 1307.8371v7.}


\bibitem[\protect\citeauthoryear{Awasthi, Balcan, Haghtalab, and Urner}{Awasthi
  et~al\mbox{.}}{2015}]%
        {ABHU15}
{P. Awasthi}, {M.-F. Balcan}, {N. Haghtalab}, {and} {R. Urner}. 2015.
\newblock \showarticletitle{Efficient Learning of Linear Separators under
  Bounded Noise}.
\newblock
\newblock
\shownote{See also Arxiv paper 1503.03594.}


\bibitem[\protect\citeauthoryear{Awasthi, Balcan, Haghtalab, and Zhang}{Awasthi
  et~al\mbox{.}}{2016}]%
        {ABHZ16}
{P. Awasthi}, {M.-F. Balcan}, {N. Haghtalab}, {and} {H. Zhang}. 2016.
\newblock \showarticletitle{Learning and 1-bit Compressed Sensing under
  Asymmetric Noise}. In {\em COLT}.
\newblock


\bibitem[\protect\citeauthoryear{Awasthi, Blum, and Sheffet}{Awasthi
  et~al\mbox{.}}{2010}]%
        {ABS10}
{P. Awasthi}, {A. Blum}, {and} {O. Sheffet}. 2010.
\newblock \showarticletitle{Improved guarantees for agnostic learning of
  disjunctions}. In {\em COLT}.
\newblock


\bibitem[\protect\citeauthoryear{Balcan, Beygelzimer, and Langford}{Balcan
  et~al\mbox{.}}{2006}]%
        {Balcan06}
{M.-F. Balcan}, {A. Beygelzimer}, {and} {J. Langford}. 2006.
\newblock \showarticletitle{Agnostic active learning}. In {\em ICML}.
\newblock


\bibitem[\protect\citeauthoryear{Balcan, Broder, and Zhang}{Balcan
  et~al\mbox{.}}{2007}]%
        {Balcan07}
{M.-F. Balcan}, {A. Broder}, {and} {T. Zhang}. 2007.
\newblock \showarticletitle{Margin based active learning}. In {\em COLT}.
\newblock


\bibitem[\protect\citeauthoryear{Balcan and Feldman}{Balcan and
  Feldman}{2013}]%
        {BF13}
{M.-F. Balcan} {and} {V. Feldman}. 2013.
\newblock \showarticletitle{Statistical Active Learning Algorithms}. In {\em
  NIPS}.
\newblock


\bibitem[\protect\citeauthoryear{Balcan and Hanneke}{Balcan and
  Hanneke}{2012}]%
        {BH12}
{M.-F. Balcan} {and} {S. Hanneke}. 2012.
\newblock \showarticletitle{Robust Interactive Learning}. In {\em COLT}.
\newblock


\bibitem[\protect\citeauthoryear{Balcan, Hanneke, and Wortman}{Balcan
  et~al\mbox{.}}{2008}]%
        {BHW08}
{M.-F. Balcan}, {S. Hanneke}, {and} {J. Wortman}. 2008.
\newblock \showarticletitle{The True Sample Complexity of Active Learning}. In
  {\em COLT}.
\newblock


\bibitem[\protect\citeauthoryear{Balcan and Long}{Balcan and Long}{2013}]%
        {BL13}
{M.-F. Balcan} {and} {P.~M. Long}. 2013.
\newblock \showarticletitle{Active and passive learning of linear separators
  under log-concave distributions}. In {\em Conference on Learning Theory}.
\newblock


\bibitem[\protect\citeauthoryear{Bartlett, Bousquet, and Mendelson}{Bartlett
  et~al\mbox{.}}{2005}]%
        {BBM05}
{P.~L. Bartlett}, {O. Bousquet}, {and} {S. Mendelson}. 2005.
\newblock \showarticletitle{Local {R}ademacher complexities}.
\newblock {\em Annals of Statistics\/} {33}, 4 (2005), 1497--1537.
\newblock


\bibitem[\protect\citeauthoryear{Beygelzimer, Hsu, Langford, and
  Zhang}{Beygelzimer et~al\mbox{.}}{2010}]%
        {nips10}
{A. Beygelzimer}, {D. Hsu}, {J. Langford}, {and} {T. Zhang}. 2010.
\newblock \showarticletitle{Agnostic Active Learning Without Constraints}. In
  {\em NIPS}.
\newblock


\bibitem[\protect\citeauthoryear{Bienstock and Michalka}{Bienstock and
  Michalka}{2014}]%
        {BM14}
{D. Bienstock} {and} {A. Michalka}. 2014.
\newblock \showarticletitle{Polynomial solvability of variants of the
  trust-region subproblem}. In {\em SODA}.
\newblock


\bibitem[\protect\citeauthoryear{Birnbaum and Shalev-Shwartz}{Birnbaum and
  Shalev-Shwartz}{2012}]%
        {BS12}
{A. Birnbaum} {and} {S. Shalev-Shwartz}. 2012.
\newblock \showarticletitle{Learning Halfspaces with the Zero-One Loss:
  Time-Accuracy Tradeoffs}. In {\em NIPS}.
\newblock


\bibitem[\protect\citeauthoryear{Blum, Frieze, Kannan, and Vempala}{Blum
  et~al\mbox{.}}{1997}]%
        {BlumFKV:97}
{A. Blum}, {A. Frieze}, {R. Kannan}, {and} {S. Vempala}. 1997.
\newblock \showarticletitle{A polynomial time algorithm for learning noisy
  linear threshold functions}.
\newblock {\em Algorithmica\/} {22}, 1/2 (1997), 35--52.
\newblock


\bibitem[\protect\citeauthoryear{Blum, Furst, Kearns, and Lipton}{Blum
  et~al\mbox{.}}{1994}]%
        {BFKL93}
{Avrim Blum}, {Merrick~L. Furst}, {Michael~J. Kearns}, {and} {Richard~J.
  Lipton}. 1994.
\newblock \showarticletitle{Cryptographic Primitives Based on Hard Learning
  Problems}. In {\em Proceedings of the 13th Annual International Cryptology
  Conference on Advances in Cryptology}.
\newblock


\bibitem[\protect\citeauthoryear{Boucheron, Bousquet, and Lugosi}{Boucheron
  et~al\mbox{.}}{2005}]%
        {bbl05}
{S. Boucheron}, {O. Bousquet}, {and} {G. Lugosi}. 2005.
\newblock \showarticletitle{Theory of Classification: a Survey of Recent
  Advances}.
\newblock {\em ESAIM: Probability and Statistics\/}  {9} (2005), 9:323--375.
\newblock


\bibitem[\protect\citeauthoryear{Bshouty, Li, and Long}{Bshouty
  et~al\mbox{.}}{2009}]%
        {BLL09}
{N.~H. Bshouty}, {Y. Li}, {and} {P.~M. Long}. 2009.
\newblock \showarticletitle{{Using the doubling dimension to analyze the
  generalization of learning algorithms}}.
\newblock {\em JCSS\/} (2009).
\newblock


\bibitem[\protect\citeauthoryear{Bylander}{Bylander}{1994}]%
        {Byl94}
{T. Bylander}. 1994.
\newblock \showarticletitle{Learning linear threshold functions in the presence
  of classification noise}. In {\em Conference on Computational Learning
  Theory}.
\newblock


\bibitem[\protect\citeauthoryear{Castro and Nowak}{Castro and Nowak}{2007}]%
        {CN07}
{R. Castro} {and} {R. Nowak}. 2007.
\newblock \showarticletitle{Minimax Bounds for Active Learning}. In {\em COLT}.
\newblock


\bibitem[\protect\citeauthoryear{{Cesa-Bianchi}, {Gentile}, and
  {Zaniboni}.}{{Cesa-Bianchi} et~al\mbox{.}}{2010}]%
        {CaCEGe10}
{N. {Cesa-Bianchi}}, {C. {Gentile}}, {and} {L. {Zaniboni}.} 2010.
\newblock \showarticletitle{Learning noisy linear classifiers via adaptive and
  selective sampling}.
\newblock {\em Machine Learning\/} (2010).
\newblock


\bibitem[\protect\citeauthoryear{Cohn, Atlas, and Ladner}{Cohn
  et~al\mbox{.}}{1994}]%
        {Cohn94}
{D. Cohn}, {L. Atlas}, {and} {R. Ladner}. 1994.
\newblock \showarticletitle{Improving Generalization with Active Learning}.
\newblock {\em Machine Learning\/} {15}, 2 (1994).
\newblock


\bibitem[\protect\citeauthoryear{Cristianini and Shawe-Taylor}{Cristianini and
  Shawe-Taylor}{2000}]%
        {Cristianini00}
{N. Cristianini} {and} {J. Shawe-Taylor}. 2000.
\newblock {\em An introduction to support vector machines and other
  kernel-based learning methods}.
\newblock Cambridge University Press.
\newblock


\bibitem[\protect\citeauthoryear{Daniely}{Daniely}{2015}]%
        {Dan15}
{A. Daniely}. 2015.
\newblock \showarticletitle{A {PTAS} for Agnostically Learning Halfspaces}. In
  {\em COLT}.
\newblock
\newblock
\shownote{See also Arxiv paper arXiv:1410.7050.}


\bibitem[\protect\citeauthoryear{Dasgupta}{Dasgupta}{2005}]%
        {sanjoy-coarse}
{S. Dasgupta}. 2005.
\newblock \showarticletitle{Coarse sample complexity bounds for active
  learning}. In {\em NIPS}.
\newblock


\bibitem[\protect\citeauthoryear{Dasgupta}{Dasgupta}{2011}]%
        {sanjoy11-encyc}
{S. Dasgupta}. 2011.
\newblock \showarticletitle{Active Learning}.
\newblock {\em Encyclopedia of Machine Learning\/} (2011).
\newblock


\bibitem[\protect\citeauthoryear{Dasgupta, Hsu, and Monteleoni}{Dasgupta
  et~al\mbox{.}}{2007}]%
        {dhsm}
{S. Dasgupta}, {D.J. Hsu}, {and} {C. Monteleoni}. 2007.
\newblock \showarticletitle{A general agnostic active learning algorithm}. In
  {\em NIPS}.
\newblock


\bibitem[\protect\citeauthoryear{Dekel, Gentile, and Sridharan}{Dekel
  et~al\mbox{.}}{2012}]%
        {dgs12}
{O. Dekel}, {C. Gentile}, {and} {K. Sridharan}. 2012.
\newblock \showarticletitle{Selective Sampling and Active Learning from Single
  and Multiple Teachers}.
\newblock {\em JMLR\/} (2012).
\newblock


\bibitem[\protect\citeauthoryear{Diakonikolas, Kane, and Stewart}{Diakonikolas
  et~al\mbox{.}}{2017}]%
        {diakonikolas2017learning}
{Ilias Diakonikolas}, {Daniel~M Kane}, {and} {Alistair Stewart}. 2017.
\newblock \showarticletitle{Learning geometric concepts with nasty noise}.
\newblock {\em arXiv preprint arXiv:1707.01242\/} (2017).
\newblock


\bibitem[\protect\citeauthoryear{Feldman, Gopalan, Khot, and
  Ponnuswami}{Feldman et~al\mbox{.}}{2006}]%
        {FGK+:06}
{V. Feldman}, {P. Gopalan}, {S. Khot}, {and} {A. Ponnuswami}. 2006.
\newblock \showarticletitle{New Results for Learning Noisy Parities and
  Halfspaces}. In {\em FOCS}. 563--576.
\newblock


\bibitem[\protect\citeauthoryear{Freund, Seung, Shamir, and Tishby.}{Freund
  et~al\mbox{.}}{1997}]%
        {QBC}
{Y. Freund}, {H.S. Seung}, {E. Shamir}, {and} {N. Tishby.} 1997.
\newblock \showarticletitle{Selective sampling using the query by committee
  algorithm.}
\newblock {\em Machine Learning\/} {28}, 2-3 (1997), 133--168.
\newblock


\bibitem[\protect\citeauthoryear{Garey and Johnson}{Garey and Johnson}{1990}]%
        {Garey90}
{Michael~R. Garey} {and} {David~S. Johnson}. 1990.
\newblock {\em Computers and Intractability; A Guide to the Theory of
  NP-Completeness}.
\newblock


\bibitem[\protect\citeauthoryear{Gonen, Sabato, and Shalev-Shwartz}{Gonen
  et~al\mbox{.}}{2013}]%
        {GSS12}
{A. Gonen}, {S. Sabato}, {and} {S. Shalev-Shwartz}. 2013.
\newblock \showarticletitle{Efficient Pool-Based Active Learning of
  Halfspaces}. In {\em ICML}.
\newblock


\bibitem[\protect\citeauthoryear{Gupta, Hardt, Roth, and Ullman}{Gupta
  et~al\mbox{.}}{2011}]%
        {Gupta11}
{Anupam Gupta}, {Moritz Hardt}, {Aaron Roth}, {and} {Jonathan Ullman}. 2011.
\newblock \showarticletitle{Privately releasing conjunctions and the
  statistical query barrier}. In {\em Proceedings of the 43rd annual ACM
  symposium on Theory of computing}.
\newblock


\bibitem[\protect\citeauthoryear{Guruswami and Raghavendra}{Guruswami and
  Raghavendra}{2006}]%
        {GR06}
{Venkatesan Guruswami} {and} {Prasad Raghavendra}. 2006.
\newblock \showarticletitle{Hardness of Learning Halfspaces with Noise}. In
  {\em Proceedings of the 47th Annual IEEE Symposium on Foundations of Computer
  Science}.
\newblock


\bibitem[\protect\citeauthoryear{Guruswami and Raghavendra}{Guruswami and
  Raghavendra}{2009}]%
        {GR09}
{V. Guruswami} {and} {P. Raghavendra}. 2009.
\newblock \showarticletitle{Hardness of Learning Halfspaces with Noise}.
\newblock {\it SIAM J. Comput.} {39}, 2 (2009), 742--765.
\newblock


\bibitem[\protect\citeauthoryear{Hanneke}{Hanneke}{2007}]%
        {Hanneke07}
{S. Hanneke}. 2007.
\newblock \showarticletitle{A Bound on the Label Complexity of Agnostic Active
  Learning}. In {\em ICML}.
\newblock


\bibitem[\protect\citeauthoryear{Hanneke}{Hanneke}{2011}]%
        {hanneke:11}
{S. Hanneke}. 2011.
\newblock \showarticletitle{Rates of Convergence in Active Learning}.
\newblock {\em The Annals of Statistics\/} {39}, 1 (2011), 333--361.
\newblock


\bibitem[\protect\citeauthoryear{Hanneke}{Hanneke}{2014}]%
        {hanneke:survey}
{S. Hanneke}. 2014.
\newblock {\em Theory of Disagreement-Based Active Learning}.
\newblock Foundations and Trends in Machine Learning.
\newblock


\bibitem[\protect\citeauthoryear{Hanneke, Kanade, and Yang}{Hanneke
  et~al\mbox{.}}{2015}]%
        {HKL15}
{S. Hanneke}, {V. Kanade}, {and} {L. Yang}. 2015.
\newblock \showarticletitle{Learning with a Drifting Target Concept}. In {\em
  ALT}.
\newblock


\bibitem[\protect\citeauthoryear{Johnson and Preparata}{Johnson and
  Preparata}{1978}]%
        {Johnson78}
{D.~S. Johnson} {and} {F. Preparata}. 1978.
\newblock \showarticletitle{The densest hemisphere problem}.
\newblock {\em Theoretical Computer Science\/} {6}, 1 (1978), 93 -- 107.
\newblock


\bibitem[\protect\citeauthoryear{Kalai, Klivans, Mansour, and Servedio}{Kalai
  et~al\mbox{.}}{2005}]%
        {KKMS08}
{Adam~Tauman Kalai}, {Adam~R. Klivans}, {Yishay Mansour}, {and} {Rocco~A.
  Servedio}. 2005.
\newblock \showarticletitle{Agnostically Learning Halfspaces}. In {\em
  Proceedings of the 46th Annual IEEE Symposium on Foundations of Computer
  Science}.
\newblock


\bibitem[\protect\citeauthoryear{Kearns and Li}{Kearns and Li}{1988}]%
        {kearns-li:93}
{Michael Kearns} {and} {Ming Li}. 1988.
\newblock \showarticletitle{Learning in the presence of malicious errors}. In
  {\em Proceedings of the twentieth annual ACM symposium on Theory of
  computing}.
\newblock


\bibitem[\protect\citeauthoryear{Kearns, Schapire, and Sellie}{Kearns
  et~al\mbox{.}}{1994}]%
        {KearnsSS94}
{Michael Kearns}, {Robert Schapire}, {and} {Linda Sellie}. 1994.
\newblock \showarticletitle{Toward Efficient Agnostic Learning}.
\newblock {\em Mach. Learn.\/} {17}, 2-3 (Nov. 1994).
\newblock


\bibitem[\protect\citeauthoryear{Kearns and Vazirani}{Kearns and
  Vazirani}{1994}]%
        {KearnsVazirani:94}
{M. Kearns} {and} {U. Vazirani}. 1994.
\newblock {\em An introduction to computational learning theory}.
\newblock MIT Press, Cambridge, MA.
\newblock


\bibitem[\protect\citeauthoryear{Klivans, Long, and Servedio}{Klivans
  et~al\mbox{.}}{2009a}]%
        {KLS09}
{A.~R. Klivans}, {P.~M. Long}, {and} {Rocco~A. Servedio}. 2009a.
\newblock \showarticletitle{Learning Halfspaces with Malicious Noise}.
\newblock {\em Journal of Machine Learning Research\/}  {10} (2009).
\newblock


\bibitem[\protect\citeauthoryear{Klivans, Long, and Tang}{Klivans
  et~al\mbox{.}}{2009b}]%
        {KLT09}
{A.~R. Klivans}, {P.~M. Long}, {and} {A. Tang}. 2009b.
\newblock \showarticletitle{{Baum}'s Algorithm Learns Intersections of
  Halfspaces with respect to Log-Concave Distributions}. In {\em RANDOM}.
\newblock


\bibitem[\protect\citeauthoryear{Koltchinskii}{Koltchinskii}{2010}]%
        {Kol10}
{V. Koltchinskii}. 2010.
\newblock \showarticletitle{Rademacher Complexities and Bounding the Excess
  Risk in Active Learning}.
\newblock {\em Journal of Machine Learning Research\/}  {11} (2010),
  2457--2485.
\newblock


\bibitem[\protect\citeauthoryear{Long and Servedio}{Long and Servedio}{2006}]%
        {LS06}
{P.~M. Long} {and} {R.~A. Servedio}. 2006.
\newblock \showarticletitle{Attribute-efficient learning of decision lists and
  linear threshold functions under unconcentrated distributions}.
\newblock {\em NIPS\/} (2006).
\newblock


\bibitem[\protect\citeauthoryear{Long and Servedio}{Long and Servedio}{2011}]%
        {LS11}
{P.~M. Long} {and} {R.~A. Servedio}. 2011.
\newblock \showarticletitle{Learning large-margin halfspaces with more
  malicious noise}. In {\em NIPS}.
\newblock


\bibitem[\protect\citeauthoryear{Lov\'asz and Vempala}{Lov\'asz and
  Vempala}{2007}]%
        {LV07}
{L. Lov\'asz} {and} {S. Vempala}. 2007.
\newblock \showarticletitle{The geometry of logconcave functions and sampling
  algorithms}.
\newblock {\em Random Structures and Algorithms\/} {30}, 3 (2007), 307--358.
\newblock


\bibitem[\protect\citeauthoryear{Monteleoni}{Monteleoni}{2006}]%
        {Monteleoni06}
{Claire Monteleoni}. 2006.
\newblock \showarticletitle{Efficient algorithms for general active learning}.
  In {\em Proceedings of the 19th annual conference on Learning Theory}.
\newblock


\bibitem[\protect\citeauthoryear{Pollard}{Pollard}{2011}]%
        {Pol84}
{D. Pollard}. 2011.
\newblock {\em Convergence of Stochastic Processes}.
\newblock


\bibitem[\protect\citeauthoryear{Raginsky and Rakhlin}{Raginsky and
  Rakhlin}{2011}]%
        {RaginskyR:11}
{M. Raginsky} {and} {A. Rakhlin}. 2011.
\newblock \showarticletitle{Lower Bounds for Passive and Active Learning}. In
  {\em NIPS}.
\newblock


\bibitem[\protect\citeauthoryear{Regev}{Regev}{2005}]%
        {Reg05}
{Oded Regev}. 2005.
\newblock \showarticletitle{On lattices, learning with errors, random linear
  codes, and cryptography}. In {\em Proceedings of the thirty-seventh annual
  ACM symposium on Theory of computing}.
\newblock


\bibitem[\protect\citeauthoryear{Servedio}{Servedio}{2001}]%
        {Ser01}
{Rocco~A. Servedio}. 2001.
\newblock \showarticletitle{Smooth Boosting and Learning with Malicious Noise}.
  In {\em 14th Annual Conference on Computational Learning Theory and 5th
  {E}uropean Conference on Computational Learning Theory}.
\newblock


\bibitem[\protect\citeauthoryear{Sturm and Zhang}{Sturm and Zhang}{2003}]%
        {SZ03}
{J. Sturm} {and} {S. Zhang}. 2003.
\newblock \showarticletitle{On cones of nonnegative quadratic functions}.
\newblock {\em Mathematics of Operations Research\/}  {28} (2003), 246--267.
\newblock


\bibitem[\protect\citeauthoryear{Valiant}{Valiant}{1985}]%
        {Valiant85}
{L.~G. Valiant}. 1985.
\newblock \showarticletitle{Learning disjunction of conjunctions}. In {\em
  Proceedings of the 9th International Joint Conference on Artificial
  intelligence}.
\newblock


\bibitem[\protect\citeauthoryear{Vapnik}{Vapnik}{1998}]%
        {vapnik:98}
{V. Vapnik}. 1998.
\newblock {\em Statistical Learning Theory}.
\newblock Wiley-Interscience.
\newblock


\bibitem[\protect\citeauthoryear{Vempala}{Vempala}{2010}]%
        {Vem10}
{S. Vempala}. 2010.
\newblock \showarticletitle{A random-sampling-based algorithm for learning
  intersections of halfspaces}.
\newblock {\em JACM\/} {57}, 6 (2010).
\newblock


\bibitem[\protect\citeauthoryear{Wang}{Wang}{2011}]%
        {wang11}
{L. Wang}. 2011.
\newblock \showarticletitle{{Smoothness, Disagreement Coefficient, and the
  Label Complexity of Agnostic Active Learning}}.
\newblock {\em JMLR\/} (2011).
\newblock


\bibitem[\protect\citeauthoryear{Zhang and Chaudhuri}{Zhang and
  Chaudhuri}{2014}]%
        {ZC14}
{C. Zhang} {and} {K. Chaudhuri}. 2014.
\newblock \showarticletitle{Beyond Disagreement-Based Agnostic Active
  Learning}.
\newblock In {\em NIPS}.
\newblock


\bibitem[\protect\citeauthoryear{Zhang}{Zhang}{2006}]%
        {Zha06}
{T. Zhang}. 2006.
\newblock \showarticletitle{Information Theoretical Upper and Lower Bounds for
  Statistical Estimation}.
\newblock {\em IEEE Transactions on Information Theory\/} {52}, 4 (2006),
  1307--1321.
\newblock


\end{thebibliography}

\appendix

\section{Additional Related Work}%
\label{se:related}
\paragraph{Passive Learning}
Blum et al. \cite{BlumFKV:97} considered noise-tolerant learning of
halfspaces under a more idealized noise model, known as the random noise model, in which the label of
each example is flipped with a certain probability, independently of
the feature vector.
Some other, less closely related, work on efficient noise-tolerant
learning of halfspaces includes \cite{Byl94,BlumFKV:97,FGK+:06,GR09,Ser01,ABS10,LS11,BS12}.

\paragraph{Active Learning}
As we have mentioned, most prior theoretical work on active learning focuses on either sample complexity bounds (without regard for efficiency) or on providing polynomial time algorithms in the noiseless case or under simple noise models (random classification noise~\cite{BF13} or linear noise~\cite{CaCEGe10,dgs12}).

In~\cite{CaCEGe10,dgs12} online learning algorithms in the selective sampling framework are presented, where labels must
be actively queried before they are revealed. Under the assumption that the label conditional distribution is a linear function determined by a fixed target vector, they provide bounds on the regret of the algorithm and on
the number of labels it queries when faced with an adaptive adversarial strategy of generating the
instances. As pointed out in \cite{dgs12}, these results can also be converted to a distributional PAC setting where instances $x_t$ are drawn i.i.d. In this setting they obtain exponential improvement in label complexity over passive learning. These interesting results and techniques are not directly comparable to ours.
One important difference is that (as pointed out in~\cite{GSS12}) the exponential improvement they give is not possible in the noiseless version of their setting. In other words, the addition of linear noise defined by the target makes the problem easier for active sampling. By contrast RCN can only make the classification task harder than in the realizable case.

Recently,~\cite{BF13} showed the first polynomial time algorithms for actively learning thresholds, balanced rectangles, and homogenous linear separators under log-concave distributions in the presence of random classification noise. 
Active learning with respect to isotropic log-concave distributions in
the absence of noise was studied in \cite{BL13}.

An algorithm for active learning with a general hypothesis space was
proposed and analyzed by Zhang and Chaudhuri \citeyear{ZC14}.
Efficient algorithms for
tracking a drifting linear classifier when the distribution is uniform
were described by Hanneke, Kanade and Yang \citeyear{HKL15}.

\section{Acute initialization}
\label{a:w0}

We will prove that we may assume without loss of generality that
the algorithm receives as input a $w_0$ whose angle with the target $w^*$ is acute.

Suppose we have an algorithm $B$ as a subroutine that satisfies the
guarantee of Theorem~\ref{t:adversarial.detailed}, given access to such
a $w_0$.  Then we can arrive at an algorithm $A$ which works without
it as follows.  With probability $1$, for a random $u$, either $u$ or
$-u$ has an acute angle with $w^*$.  We may then run $B$ with both
choices, and with $\epsilon$ set to $\frac{\pi c_2}{4}$, where $c_2$
is the constant in Part (e) of Lemma~\ref{l:ilc}.  Then we can use
hypothesis testing on $O(\log(1/\delta))$ examples, and, with high
probability, find a hypothesis $w'$ with error less than $\frac{\pi
c_2}{4}$.
Part (e) of Lemma~\ref{l:ilc} then implies
that $A$ may then set $w_0 = w'$, and call $B$ again.

\section{Relating Adversarial Label Noise and the Agnostic Setting}
\label{a:agnostic}
In this section we study the agnostic setting of~\cite{KearnsSS94,KKMS08} and describe how our results imply constant factor approximations in that model. In the agnostic model, data $(x,y)$ is generated from a distribution $D$ over $\Re^d \times \{1,-1\}$. For a given concept class $C$, let $OPT$ be the error of the best classifier in $C$. In other words, $OPT = \mathrm{argmin}_{f \in C} err_D(f) = \mathrm{argmin}_{f \in C} Pr_{(x,y) \sim D} [f(x) \ne y]$. The goal of the learning algorithm is to output a hypothesis $h$ which is nearly as good as $f$, i.e., given $\epsilon > 0$, we want $err_D(h) \leq c \cdot OPT + \epsilon$, where $c$ is the approximation factor.
Any result in the adversarial model that we study, translates into a result for the agnostic setting via the following lemma.
\begin{lemma}
\label{lem:adversarial-to-agnostic}
For a given concept class $C$ and distribution $D$, if there exists an algorithm in the {\em adversarial} noise model which runs in time $\poly(d,1/\epsilon)$ and tolerates a noise rate of $\eta = \Omega(\epsilon)$, then there exists an algorithm for $(C,D)$ in the agnostic setting which runs in time $\poly(d,1/\epsilon)$ and achieves error $O(OPT + \epsilon)$.
\end{lemma}
\begin{proof}
Let $f^*$ be the optimal halfspace with error $OPT$. In the adversarial setting, w.r.t. $f^*$, the noise rate $\eta$ will be exactly $OPT$. Set $\epsilon' = c (OPT + \epsilon)$ as input to the algorithm for the adversarial model. By the guarantee of the algorithm we will get a hypothesis $h$ such that $Pr_{(x,y) \sim D} [h(x) \ne f^*(x)] \leq \epsilon' = c(OPT + \epsilon)$. Hence by triangle inequality, we have $err_D(h) \leq err_D(f^*) + c(OPT + \epsilon) = O(OPT + \epsilon)$.
\end{proof}
For the case when $C$ is the class of origin centered halfspaces in $R^d$ and the marginal of $D$ is the uniform distribution over $S_{d-1}$, the above lemma along with Theorem~\ref{thm:log-concave-agnostic} implies that we can output a halfspace of accuracy $O(OPT+\epsilon)$ in time $\poly(d,1/ \epsilon)$. The work of~\cite{KKMS08} achieves a guarantee of $O(OPT + \epsilon)$ in time exponential in $1/\epsilon$ by doing $L_2$ regression to learn a low degree polynomial, and that $L_1$ regression can achieve a stronger guarantee of $OPT + \epsilon$.
As noted above, their approach also does not require that the
halfspace to be learned goes through the origin.
\section{Proof of VC lemmas}
\label{a:vc}

In this section, we apply some standard VC tools to establish
some lemmas about estimates of expectations.

\begin{definition}
Say that a set $F$ of real-valued functions with a common domain
$X$ {\em shatters} $x_1,...,x_d \in X$ if there are thresholds
$t_1,...,t_d$ such that
\[
\{ (\sign(f(x_1) - t_1),...,\sign(f(x_d) - t_d)): f \in F \}
  = \{ -1, 1 \}^d.
\]
The {\em pseudo-dimension} of F is the size of the largest
set shattered by $F$.
\end{definition}

We will use the following bound.
\begin{lemma}[see \cite{AB99}]
\label{l:vc}
Let $F$ be a set of functions from a common domain $X$ to
$[a,b]$ and let $d$ be the pseudo-dimension of $F$, and let
$D$ be a probability distribution over $X$.  Then,
for $m = O\left(\frac{(b-a)^2}{\alpha^2} (d + \log(1/\delta))\right)$,
if $x_1,...,x_m$ are drawn independently at random according to
$D$, with probability $1 - \delta$, for all $f \in F$,
\[
\left| \E_{x \sim D}(f(x)) - \frac{1}{m} \sum_{t=1}^m f(x_t) \right|
  \leq \alpha.
\]
\end{lemma}

\subsection{Proof of Lemma~\ref{l:vc.hinge}}
\label{a:vc.hinge}

The pseudo-dimension of the set of linear combinations of $d$
variables is known to be $d$ \cite{Pol84}. Since, for any
non-increasing function $\psi: \R \rightarrow \R$ and any
$F$, the pseudo-dimension of $\{ \psi \circ f: f \in F \}$
is at most that of $F$ (see \cite{Pol84}), the pseudo-dimension of
$\{ \ell(w,\cdot): w \in \R^d \}$ is at most $d$.

Now, to apply Lemma~\ref{l:vc}, we want an upper bound on the loss.
The first step is a bound in terms of the norm.
\begin{lemma}
\label{l:worst}
There is a constant $c$ such that,
for any $w \in B({w}_{k-1},r_k)$, and all $x$,
\[
\ell(w,x,y) \leq c (1 + || x ||_2).
\]
\end{lemma}
\begin{proof}
\begin{align*}
\ell(w,x,y) & \leq 1 + \frac{|w \cdot x|}{\tau_k}
 \leq 1 + \frac{|w_{k-1} \cdot x| + \Vert w-w_{k-1} \Vert_2 || x ||_2}{\tau_k} \\
& \leq 1 + \frac{b_{k-1} + r_k || x ||_2}{\tau_k}
 = 1 + \frac{c_1' {M}^{-k} + \min \{ {M}^{-(k-1)}/c_6, \pi/2 \}
                           || x ||_2}{\frac{c_2 \min \{ c_1' {M}^{-k}, c_1 \} \kappa}{6 c_3}}.
\end{align*}
\end{proof}

If the support of $D$ is bounded, Lemma~\ref{l:worst} gives a useful
worst-case bound on the loss.  Next, we give a high-probability bound
that holds for all isotropic log-concave distributions.
\begin{lemma}
\label{l:x.small.highprob}
For an absolute constant $c$,
with probability $1 - \frac{\delta}{6 (k + k^2)}$,
\begin{equation}
\label{e:ballbound}
\max_{x \in W_C} || x ||_2 \leq c \sqrt{d} \ln \left( \frac{|W_C| k}{\delta} \right).
\end{equation}
\end{lemma}
\begin{proof}
Applying Part (a) of Lemma~\ref{l:ilc}
together with a union bound, we have
\[
\Pr( \exists x \in W_C,\; || x || > \alpha)
\leq c_9 | W_C | \exp(-\alpha/\sqrt{d}),
\]
and $\alpha = \sqrt{d} \ln \left( \frac{12 c_9 |W_C| k^2}{\delta} \right)$
makes the RHS at most $\frac{\delta}{6 (k + k^2)}$.
\end{proof}

Let $D'$ be the distribution obtained by conditioning $D$ on
the event that $|| x || < R$,
where $R$ is the RHS of (\ref{e:ballbound}).
By Lemma~\ref{l:x.small.highprob}, the total variation distance between
drawing the members of $W_C$ independently random from $D$, and drawing
them from $D'$, it at most $1 - \frac{\delta}{6 (k + k^2)}$,
so it suffices to prove
(\ref{e:vc.labelnoise}) with respect
to $D'$.
Applying Lemma~\ref{l:worst},
and Lemma~\ref{l:vc}
then completes the proof of (\ref{e:vc.labelnoise}).

\subsection{Proof of Lemma~\protect\ref{lemma:all.var}}
Define $f_a$ by $f_a(x) = (a \cdot x)^2$.  The
pseudo-dimension of the set of all such functions is $O(d)$
\cite{KLS09}.  As the proof of Lemma~\ref{l:vc.hinge},
w.l.o.g., all $x$ have
$|| x ||_2 \leq O(\sqrt{d} \log(\ell/\delta))$,
and applying Lemma~\ref{l:vc} completes the proof.

\end{document}